\documentclass[twoside,11pt]{article}

\usepackage{blindtext}

\usepackage[abbrvbib, preprint]{jmlr2e}

\usepackage{jmlr2e}

\usepackage{lastpage}

\ShortHeadings{Kernel Approximation of Gradient Flows}{Zhu and Mielke}
\firstpageno{1}

\usepackage{bef_alex} %

\newcommand{\EEE}{\color{black}}

\usepackage{preamble}
\usepackage{enumitem} %

\renewenvironment{proof}[1][Proof]%
                  {\par\noindent{\bfseries #1.} }{\rule{1ex}{1ex}\\[2mm]}

\newcommand{\DFRk}{\delta\text{-}\FR k}

\numberwithin{equation}{section}
\numberwithin{theorem}{section}
\numberwithin{example}{section}

\usepackage{tikz}
\usetikzlibrary{arrows} %

\begin{document}

\title{
Kernel Approximation of Fisher-Rao
Gradient Flows
}

\author{\name{Jia-Jie Zhu} \email{zhu@wias-berlin.de}\\
\addr
Weierstrass Institute for Applied Analysis and Stochastics, Berlin, Germany\\
\AND
  \name{Alexander Mielke} \email{mielke@wias-berlin.de} \\
  \addr
  Humboldt University of Berlin \&
Weierstrass Institute for Applied Analysis and Stochastics\\
Berlin, Germany
       }

\editor{My editor}

\maketitle

\begin{abstract}%
The purpose of this paper is to answer a few open questions in the interface of kernel methods and PDE gradient flows.
Motivated by recent advances in machine learning, particularly in generative modeling and sampling, we present a rigorous investigation of
Fisher-Rao and Wasserstein type gradient flows concerning
their gradient structures, flow equations, and their kernel approximations.
Specifically, we focus on the Fisher-Rao (also known as Hellinger) geometry and its various kernel-based approximations, developing a principled theoretical framework using tools from PDE gradient flows and optimal transport theory. We also provide a complete characterization of gradient flows in the maximum-mean discrepancy (MMD) space, with connections to existing learning and inference algorithms. Our analysis reveals precise theoretical insights linking Fisher-Rao flows, Stein flows, kernel discrepancies, and nonparametric regression. We then rigorously prove evolutionary $\Gamma$-convergence for kernel-approximated Fisher-Rao flows, providing theoretical guarantees beyond pointwise convergence. Finally, we analyze energy dissipation using the Helmholtz-Rayleigh principle, establishing important connections between classical theory in mechanics and modern machine learning practice. Our results provide a unified theoretical foundation for understanding and analyzing approximations of gradient flows in machine learning applications through a rigorous gradient flow and variational method perspective.

\end{abstract}

\tableofcontents

\section{Introduction}
\label{sec:intro}
In essence, the seminal works such as \citep{ottoGeometryDissipativeEvolution2001,jordan_variational_1998}
enabled a systematic perspective of studying the PDE such as the type
\begin{align*}
    \partial_t \mu = - \DIV \left( \mu \nabla\dFdmu \right) 
\end{align*}
as gradient flows of the energy functional $F$, i.e., the solution paths of the measure optimization problem
\begin{align}
    \min _{\mu \in \mathcal{P}(\mathbb{R}^{d})} F(\mu)
    \label{eq:measure-opt}
\end{align}
in the Wasserstein space $\left(\mathcal{P}(\mathbb{R}^{d}), W_p\right)$.
While much recent research in
machine learning applications
has predominantly focused on modifying the \emph{energy functionals} (\ie $F$ above) of the pure Wasserstein gradient flows, e.g., \citep{arbel_maximum_2019,chewiSVGDKernelizedWasserstein2020,glaserKALEFlowRelaxed2021,korbaKernelSteinDiscrepancy2021,carrillo2019blob,lu2023birth,javanmard_analysis_2020,craig_blob_2023}, we delve into various kernel approximations of the gradient flow \emph{geometry}.
In doing so, our investigation also reveals precise relations between various previously proposed geometries over probability measures, such as the Stein distance~\citep{duncan2019geometry,liu_stein_2019},
kernel Stein discrepancy~\citep{liu_kernelized_2016,gorham2017measuring,chwialkowskiKernelTestGoodness2016},
Sobolev discrepancy~\citep{mroueh_unbalanced_2020}, and maximum mean discrepancy~\citep{gretton2012kernel}.
By working with the different gradient flow geometries and leaving the energy functional to be chosen for specific applications, we provide a measure optimization~\eqref{eq:measure-opt} framework to be adapted to various applications beyond the confines of ad-hoc energy functionals,
evidenced in the following examples.

\paragraph*{Inference and sampling in Stein geometry}
Suppose a statistician wishes to generate samples from a probability distribution $\pi$, whose density is in the form
$\displaystyle\pi(x) = \frac{1}{\int \rme^{-V(x)}\dd x}e^{-V(x)}$,
where $ V$ is the potential energy.
In this case, using the fact that $\pi$ is the invariant distribution of the Langevin stochastic differential equation
\begin{equation} 
    \mathrm{d}X_t = -{\nabla}V(X_t) \ts \mathrm{d}t + \sqrt{2} \ts \mathrm{d}Z_t
    ,
    \label{eq:langevin}
\end{equation}
where $Z_t$ is the standard Brownian motion.
From a PDE perspective, this Langevin SDE~\eqref{eq:langevin} describes the same dynamical system as the deterministic drift-diffusion Fokker-Planck PDE
\begin{align}
    \partial_t \mu =-\DIV \left(\mu \nabla \left(V + \log\mu\right) \right) 
    \label{eq:drift-diffusion-PDE-intro}
\end{align}
for probability measure $\mu$, which is also the gradient-flow equation of a Wasserstein gradient flow. 
Therefore, instead of relying on the stochastic simulation of \eqref{eq:langevin},  
one can forward simulate the deterministic PDE~\eqref{eq:drift-diffusion-PDE-intro}.
\citet{liu_stein_2019} have proposed a deterministic discrete-time update algorithm referred to as Stein variational gradient descent (SVGD).
This algorithm has been related to the Stein PDE by \citet{duncan2019geometry}
\begin{align}
    \partial_t \mu =-\DIV \left(\mu \K_\mu\nabla \left(V + \log\mu\right) \right) 
    \label{eq:stein-PDE}
\end{align}
where $\K_\mu$ is the integral operator.
The gradient flow equation~\eqref{eq:stein-PDE}
can be viewed as
the kernel approximation (or kernelization) of the 
pure Wasserstein gradient flow equation~\eqref{eq:drift-diffusion-PDE-intro}.

\paragraph*{Deep generative models}
A recent application of gradient flows is generative models.
One particular relevant class of algorithms is the 
score-based deep diffusion generative models~\citep{song2020score,song_generative_2020,ho_denoising_2020,sohl-dicksteinDeepUnsupervisedLearning2015,debortoliConvergenceDenoisingDiffusion2023,oko2023diffusion}.
The goal of the so-called \emph{score-matching} task is to compute the vector field $\nabla\log\mu_t$ to simulate a backward SDE
\begin{equation} 
    \mathrm{d}X_t = \left(X_t +2\nabla \log\mu_t\left(X_t\right)\right) \ts \mathrm{d}t + \sqrt{2} \ts \mathrm{d}W_t,
    \label{eq:reverse-sde-diffusion-model}
\end{equation}
The term $\nabla \log\mu_t\left(X_t\right)$ can be approximated via
regression in practice,
\begin{align}
\inf_{f\in \calF}
             \int_0^T
    \|f(\cdot , t) - \nabla \log\mu_t\|_{L^2_{\mu_t}}^2 
    \dd t
    ,
    \label{eq:score-matching}
\end{align}
where $\mu_t$ is the state distribution of a diffusion process at time $t$, e.g., Ornstein–Uhlenbeck process.
    Another class of generative models that has shown
    improved efficiency and stability
    is the flow-based generative models
    \citep{lipman2022flow}.
    They learn the solution $u$ to the ODE
$        \dot u = - v_t(u)$
    for some velocity field $v_t$ by solving the regression problem
    with explicit target velocity $v_t$
    \begin{align}
             \inf_{f\in\calF}
             \int_0^T
                \|f(\cdot , t) - v_t \|_{L^2_{\mu_t}}^2  
                \dd t
                .
                \label{eq:flow-based-regression}
    \end{align}
    Furthermore, they observed that, by choosing $v_t$ to be the velocity field for the optimal transport between Gaussian distributions, they obtained more efficient and stable training than previous generative models.
    Note that in practice, $f$ is often parameterized using a time-dependent neural network
and training is further done over various initial conditions\footnote{
    For pactical implementation, the target can be additionally conditioned on a variable $z\sim p$, i.e., $\xi(x, z)$, resulting in the conditional flow matching~\citep{lipman2022flow,tongImprovingGeneralizingFlowbased2023}
    \begin{align}
        \argmin_{f\in \calF}
        \left\{
            \int_0^T
            \int \int \left(f(x, z) - \xi (x, z)\right)^2 \dd \mu(x|z) p(z) \dd z
        +
        \lambda \|f\|^2_{\calF}
        \dd t
        \right\}
        .
    \end{align}    
For notational simplicity, we only consider the unconditional case.}.

    The hope of such learning algorithms
    is to approximate some vector field
    of the original target flows.
    From this paper's perspective, it is crucial to note that the \emph{new flow following the learned velocity field, in the above formulations, has a new geometry that is different from the original flow even if the training error is close to zero}.
    While there is a growing body of work on the various applications of gradient flows in machine learning, there has been no principled framework to understand the approximation of PDE gradient flows, i.e., the precise characterization of the gradient structures of the gradient systems that generate the flows.
    The goal of this paper is then to provide the theoretical foundation for such approximations to the gradient-flow geometry, mainly focusing on the Fisher-Rao as well as the Wasserstein dissipation geometry.

An important topic of this paper is
the geodesics and gradient flows in the Fisher-Rao geometry,
generated by the
Fisher-Rao distance,
also known as the Hellinger distance,
between two nonnegative
measures $\mu, \nu\in \Mplus$.
It is defined as
\begin{align}
    \label{eq:fr-def}
    \FR^2(\mu,\nu)
    =
    4\cdot \int \left(\sqrt{\frac{\delta \mu}{\delta \gamma}}
    - \sqrt{\frac{\delta \nu}{\delta \gamma}} \right)^2 \dd \gamma
\end{align}
for a reference measure 
$\gamma,\ \mu,\nu<<\gamma$.
Recall its dynamic formulation
\footnote{One may replace the dynamics by
$\dot \mu = -  \mu\cdot \xi_t + \mu\cdot \int \mu  \xi_t$ for a flow over probability measures, \ie spherical Hellinger~\citep{LasMie19GPCA}, instead of non-negative measures.
We defer the spherical variant to a future work.
Furthermore, as detailed in Remark~\ref{rem:hellinger-vs-fr},
the naming of ``Fisher-Rao'' is an abuse of the naming convention.
The accurate name of the distance defined in \eqref{eq:fr-def} is ``Hellinger''.
The Fisher-Rao distance is equal to the Bhattacharyya divergence up to a constant factor.
\label{footnote:spherical-FR-normalize}
}
(see also \citep{gallouet2017jko}, \citep{liero_optimal_2018} and
Example~\ref{ex:Fisher-Rao geodesics})
\begin{equation}
\label{eq:bb-formula-fr}
    \begin{aligned}
    \FR^2(\mu,\nu)
    =
    \min_{\mu, \xi_t} 
    \left\{\int_0^1
    \| \xi_t\|^2_{L^2_\mu}
    \dd t
    \ \
    \middle \vert \ \  
    \dot \mu = -  \mu\cdot \xi_t 
    ,
    \
    \mu(0) = \mu,
    \
    \mu(1)= \nu\right\}
    .
    \end{aligned}
\end{equation}

Using the tools from kernel methods and the Benaou-Brenier dynamic formulation,
we first investigate new gradient systems centered around the Fisher-Rao geometry.
Our emphasis is on the precise relation in terms of the kernelization of gradient flows, stated in Definition~\ref{def:f-v-kernelize}.
Motivated by the kernelization in the Stein geometry (see, \eg \eqref{eq:stein-PDE}),
we provide the geodesic and gradient structures of the \emph{kernelized Fisher-Rao} geometry in Section~\ref{sec:fisher-rao}, whose gradient flow equation is
a reaction equation with a kernelized growth field
\begin{align*}
    \dot \mu = - \mu \K_\mu  \dFdmu
    .
\end{align*}
Furthermore, we find that the kernelization of the kernel maximum-mean discrepancy, commonly used in machine learning applications, results in precisely the pure Fisher-Rao geometry. We summarize the relations below.
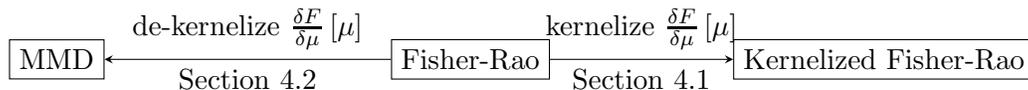
\begin{center}
\begin{tikzpicture}[node distance=5.5cm,
    ->,
    >=stealth]
    \node[draw] (A) {MMD};
    \node[draw, right of=A] (B) {Fisher-Rao};
    \node[draw, right of=B] (C) {Kernelized Fisher-Rao};
    
    \draw[<-] (A) -- (B) node[midway, above] {de-kernelize $\dFdmu$} node[midway, below] {Section~\ref{sec:mmd-reg}};
    \draw[->] (B) -- (C) node[midway, above] {kernelize $\dFdmu$} node[midway, below] {Section~\ref{sec:kern-FR}};
    \label{fig:kernelization-relations-FR}
\end{tikzpicture}
\end{center}
The arrows denote kernelization operations of the gradient flow by the operator $\K_\rho^\frac12$, as in Definition~\ref{def:f-v-kernelize}.
This is analogous to the Wasserstein setting, where we first continue the work of \cite{duncan2019geometry} on Stein geometry, of which we briefly provide the gradient structure in Section~\ref{sec:w2stein}.
We then establish the kernelization and de-kernelization relation in the diagram below.
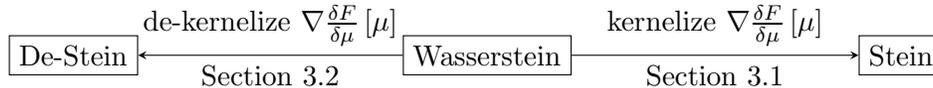
\begin{center}
    \begin{tikzpicture}[node distance=5.5cm,
        ->,
        >=stealth]
        \node[draw] (A) {De-Stein};
        \node[draw, right of=A] (B) {Wasserstein};
        \node[draw, right of=B] (C) {Stein};
        
        \draw[<-] (A) -- (B) node[midway, above] {de-kernelize $\nabla\dFdmu$} node[midway, below] {Section~\ref{sec:dk-w2}};
        \draw[->] (B) -- (C) node[midway, above] {kernelize $\nabla\dFdmu$} node[midway, below] {Section~\ref{sec:w2stein}};
        \label{fig:kernelization-relations-WS}
    \end{tikzpicture}
    \end{center}
For example, we de-kernelize the Wasserstein geometry to obtain the De-Stein distance, which results in a flat distance in the form of a static definition:
\begin{align*}
    \operatorname{De-Stein}^2(\mu,\nu)
    =    \sup_\zeta
    \left\{
        \int \zeta\dd (\mu - \nu)
        -\frac14
         \|\nabla \zeta\|^2_{\rkhs}
    \right\}
    ,
\end{align*}
which is the transport analog of the MMD.
In the appendix,
we also uncover a few interesting kernel discrepancy functionals that are interaction energies obtained by dissipating entropy along the kernelized
FR and Stein flow.
Prominent cases include the MMD~\citep{gretton2012kernel}
and the KSD~\citep{liu_kernelized_2016}.
They are both generated by the dissipation in kernelized geometries
of the inclusive KL divergence.

Motivated by the approximation of velocity fields for generative model applications,
we investigate the gradient flow structure that generates the reaction equation
whose growth field is obtained by nonparametric regression, e.g., the kernel ridge regression,
\begin{align}
    \dot {\mu_t} = - {\mu_t}\cdot r_t , \quad
        r_t = \argmin_{f\in \calF}
        \biggl\{
                \left\|f -  \dFdmut\right\|^2_{L^2_{\mu_t}}  
                + 
            \lambda \| f \|^2_{\calF}
            \biggr\}
            .
            \label{eq:gfe-krr-FRk-npreg}
    \end{align}
Note that the nonparametric regression can be cast in the more general maxium likelihood estimation (MLE) form~\eqref{eq:gfe-krr-FRk-MLE}.
    We also refer to the (gradient) system that generates the equation above as an \emph{approximate} Fisher-Rao (gradient) system.

The connection between nonparametric regression
and gradient flows
can be seen in the
Helmholtz-Rayleigh Principle~\citep{rayleigh_general_1873}; see Section~\ref{sec:nonparametric-regression}.
The intuition is that
the variational problem
\begin{align*}
        \min
        \biggl\{{\textrm{energy}}
            +{\textrm{dissipation potential}}\biggr\}
\end{align*}
is equivalent to the nonparametric regression formulation~\eqref{eq:gfe-krr-FRk-npreg}
agnostic of the dissipation geometry,
\ie
in both the Fisher-Rao and Wasserstein settings.
This implies different training objectives,
such as 
matching the score function or the log density,
can be unified using the same formalism
of the Rayleigh Principle.

Finally,
we are motivated by the
following questions
regarding
regression problems that appeared in generative models, \eg the flow-matching problem~\eqref{eq:flow-based-regression} and the score-matching problem~\eqref{eq:score-matching}. For those problems, does the learned flow exist, is it a gradient flow? If so, what is the gradient structure, e.g., energy and dissipation geometry?
Furthermore,
in nonparametric regression~\eqref{eq:gfe-krr-FRk-npreg},
one typically
bounds the prediction error, \ie
quantities such as $\|r_t - \log\mu_t\|_{L^2_{\mu_t}}$ for a fixed time $t$.
However, in gradient flows,
we are also interested in the behavior of the system
that follows the flow
$\dot {\mu_t} = - {\mu_t}\cdot r_t $,
\eg
its variational structure, solution existence, and convergence behavior.
For those reasons,
we establish
the
evolutionary $\Gamma$-convergence
in the kernel-approximate Fisher-Rao geometries,
\ie we use tools from the calculus of variations and functional analysis
instead of statistical bounds.
In a nutshell, we establish
\begin{align*}
    \text{the system generating $\dot {\mu_t} = - {\mu_t}\cdot r_t $}
    \Gto \text{Fisher-Rao gradient system}.
\end{align*}
Thus, we provide a
rigorous justification for the approximation
using nonparametric regression in the Fisher-Rao flows.
This M-estimation perspective also differs from the perspective of local regression and local smoothing
analyzed in the works by, \eg \citet{lu2023birth,carrillo2019blob}.
\begin{remark}
    [``Hellinger'' versus ``Fisher-Rao'']
    In this paper, we have used the popular naming of ``Fisher-Rao''
    to describe the infinite-dimensional geometry over probability and
    non-negative measures.
    Our main reason of doing so is the familiarity of the name in the machine learning community.
    However,
    based on the historical development of the theory,
    this is an abuse of the naming convention.
    The accurate name for this geometry is ``Hellinger''~\citep{hellingerNeueBegrundungTheorie1909,kakutaniEquivalenceInfiniteProduct1948,rao1945information,bhattacharyyaMeasureDivergenceTwo1946}.
    The Fisher-Rao naming should be instead used for the finite-dimensional parameterization of the probability measures.
    Furthermore, another nuance is that there has been a mix-up of Fisher-Rao flows over the probability measures $\calP$ versus the positive measures $\Mplus$.
    As \citet{chizat_unbalanced_2019} referred to the positive measure version as the ``Fisher-Rao geometry''.
    Yet authors such as \citet{lu2023birth,chenGradientFlowsSampling2023}
    have used the name for the probability measure version.
    To avoid confusion, in this paper, when we say ``Fisher-Rao geometry'', we always refer to the positive measure geometry.
    \label{rem:hellinger-vs-fr}
\end{remark}

\paragraph*{Organization of the paper}
In Section~\ref{sec:background}, 
we provide background on gradient systems and optimal transport, with a focus on dynamic formulation and geodesics.
Then, we recall some facts about the reproducing kernel Hilbert spaces (RKHS).
Those two topics are married through our concrete investigation in
the next four sections.
In Section~\ref{sec:wasserstein}, \ref{sec:fisher-rao}, and also \ref{sec:wfr},
we provide gradient structures for a few new and existing gradient systems of interest.
They are motivated by two types of geometries,
namely, the Fisher-Rao and Wasserstein space.
We characterize their precise relations with other kenelized and approximating geometries.
Then, we analyze the approximation quality,
by proving the evolutionary $\Gamma$-convergence of
the kernel-approximate Fisher-Rao gradient systems.
Additional proofs are given in Section~\ref{sec:proof}.

\EEE

\paragraph*{Notation}
We use the notation $\mathcal{P}(\bar{\Omega}), \Mplus(\bar{\Omega})$ to denote the space of probability and non-negative measures on the closure of a set $\Omega\subset\R^d$.
The base space symbol $\Omega$ is often dropped if there is no ambiguity in the context.
We express the standard integral operator associated with a measure $\rho$ and a kernel $k$ as a weighted convolution
$\K_\rho:L^2(\rho) \to L^2(\rho), f \mapsto \int k(x, \cdot ) f(x) \dd \rho(x)$;
cf. Theorem~\ref{thm:rkhs-int-operator}.
The measure $\rho$ is omitted if it is the Lebesgue measure, \ie $\K$.
In this paper, the first variation of a functional $F$ at $\mu\in\Mplus$ is defined as a function ${\frac{\delta F}{\delta\mu}[\mu] }$
\begin{align}
    \frac{\dd }{\dd \epsilon}F(\mu + \epsilon \cdot v) |_{\epsilon=0}
    = \int {{\frac{\delta F}{\delta\mu}[\mu] }}(x) \dd v (x)
    \label{eq:first-var-def}
\end{align}
for any perturbation in measure $v$ such that $\mu + \epsilon \cdot v\in \Mplus$. 
The Fr\'echet (sub-)differential on a Banach space $(X, \|\|_X)$ is defined as a set in the dual space
$$
\mathrm{D}^X F :=\left\{\xi \in X^{*} \mid {F}(\mu) \geq {F}_\nu+\langle\xi, \mu-\nu\rangle_{X}+o\left(\|\mu-\nu\|_{X}\right) \text { for }\mu\to\nu \right\}
,
$$
where the small-$o$ notation signifies that the term vanishes more rapidly than the term inside the parentheses.
We use superscripts for differential derivatives to emphasize the corresponding space of those operations, \ie we distinguish between
$\mathrm{D}^X F$ and $\mathrm{D}^{Y} F$.
For simplicity, we carry out the Fenchel-conjugation calculation in the un-weighted $L^2$ space.
Replacing that with duality pairing in the weighted $L^2_{\rho}$ space does not alter the results.
Common acronyms, such as partial differential equation
(PDE) and
integration by parts (IBP), are used without further specifications.
We often omit the time index $t$ to lessen the notational burden, \eg the measure at time $t$, $\mu(t, \cdot)$, is written as $\mu$.
The infimal convolution (inf-convolution) of two functions $f,g$ on Banach spaces is defined as
    $(f\square g)(x) = \inf_{y} \left\{f(y) + g(x-y)\right\}$.
In formal calculation,
we often use measures and their density interchangeably,
\ie$\int f\cdot \mu$ means the integral w.r.t. the measure $\mu$.
This is based on the
standard rigorous generalization from flows over continuous measures to discrete measures \citep{ambrosio2008gradient}.

\section{Preliminaries}
\label{sec:background}

\subsection{Gradient-flow systems and geodesics}
\label{sec:grad-flow-geod}
Intuitively, a gradient flow describes a dynamical system that is driven towards the fastest dissipation of certain energy,
through a geometric structure measuring dissipation. In this work, we restrict to the case that the dissipation law is linear, which means it can be given in terms of a (pseudo) Riemannian metric. Such a system is called a \emph{gradient system}.
For example, the dynamical system described by 
an ordinary differential equation in the Euclidean space
\(\dot u(t) = - \nabla F(u(t)), u(t)\in \mathbb R^d \)
is a simple gradient system.

In this paper, we take the perspective of variational modeling and principled mathematical analysis, i.e., we study the underlying dynamical systems modeled as gradient systems specified by the underlying space $X$, energy functional $F$, and the dissipation geometry specified by the potential $\calR$.
Given a smooth state space $X$, a dissipation potential is a function on the tangent bundle $\rmT X$, i.e.\ $\calR= \calR(u,\dot u)$, where, for all $u\in X$, the functional $\calR(u,\cdot)$ is non-negative, convex, and satisfies $\calR(u,0)=0$. We denote by 
\begin{align}
\calR^*(u,\xi) = \sup\bigset{\langle \xi, v\rangle - \calR(u,v)}{ v \in \rmT_u X} 
\label{eq:dual_dissipation_potential}
\end{align}
the (partial) Legendre transform of $\calR$ and call it the \emph{dual dissipation potential}.
Throughout this work, we will restrict to the case that $\calR(u,\cdot)$ is quadratic, i.e.\ 
\[
\calR(u,\dot u) = \frac12 \langle \bbG(u)\dot u, \dot u\rangle 
\quad \text{or equivalently} \quad \calR^*(u,\xi)=\frac12\langle \xi, \bbK(u) \xi\rangle
.
\]

\begin{definition}[Gradient system]
\label{def:GradSystem}
A triple $(X, F, \calR)$ is called a generalized gradient system, if $X$ is a manifold or a subset of a Banach space, $F:X\to \R$ is a differentiable function, and $\calR$ is a dissipation potential. The associated gradient-flow equation has the primal and dual form 
\begin{equation}
  \label{eq:GFE}
0=\rmD_{\dot u}\calR(u,\dot u) + \rmD F(u) \quad \Longleftrightarrow \quad 
 \dot u= \rmD_{\xi} \calR^*\big(u, {-}\rmD F(u)\big). 
\end{equation}
If $\calR$ is quadratic, we simply call $(X,F,\calR)$ a \emph{gradient system} and obtain the gradient flow equations
     \[
0=\bbG(u)\dot u + \rmD F(u) \quad \Longleftrightarrow \quad 
 \dot u= - \bbK(u)\rmD F(u). 
    \]  
$\bbG=\bbK^{-1}$ is called the Riemannian tensor, and $\bbK= \bbG^{-1}$ is called the Onsager operator.   
\end{definition}
\EEE

Of particular interest to this paper is the gradient flow in the Fisher-Rao metric space, also called \emph{Hellinger-Kakutani} or simple \emph{Hellinger space} in \citep{liero_optimal_2018, LasMie19GPCA}, which is the gradient system that generates the 
following reaction equation as its \emph{gradient flow equation}
\begin{align}
  \partial_t\mu = - \mu \cdot \dFdmu,
\end{align}
where $\dFdmu$ is the first variation~\eqref{eq:first-var-def}.
Alternatively, one can also view the whole r.h.s as the Fisher-Rao metric gradient using the weighted tangent space $L^2_\mu$.

The Fisher-Rao gradient system is a special case of general gradient flow in metric spaces \cite{ambrosio2008gradient}, which has gained significant attention in recent machine learning literature due to the study of Wasserstein gradient flow (WGF), originated from the seminal works of Otto and colleagues, e.g., \cite{otto1996double,jordan_variational_1998,ottoGeometryDissipativeEvolution2001}.
Rigorous characterizations of general metric gradient systems have been carried out in PDE literature, for which we refer to \cite{ambrosio2008gradient} for complete treatments and \cite{santambrogio_optimal_2015, peletier_variational_2014,mielke2023introduction} for a first principles' introduction.
To get a concrete intuition, the gradient structure of the following two classical PDEs will become relevant in later discussions about Fisher-Rao and Wasserstein respectively. 
\begin{example}
    [Classical PDE: Allen-Cahn and Cahn-Hilliard]
    Recall the Allen-Cahn PDE
    \begin{align}
        \partial_t \mu = \Delta \mu - \nabla V,
        \label{eq:allen-cahn}
    \end{align}
    and the Cahn-Hilliard PDE
\begin{align}
    \partial_t \mu = \Delta \left( -\Delta\mu + \nabla V \right)
    .
    \label{eq:cahn-hilliard}
\end{align}
They are the gradient flows of the
energy functional
$F(\mu) = \frac12\int |\nabla \mu|^2  + \int V(\mu) $ in two different Hilbert spaces, where $V$ is a potential function, e.g., the double-well potential $V(x) = \frac14 (1-x^2)^2$.
Allen-Cahn
is the Hilbert-space gradient flow of the energy
$F$
in unweighted $L^2$, \ie
\begin{align}
    \mathcal R_\mathrm{AC} (\mu, \dot u) = \frac12\|\dot u\|_{L^2}^2
,
\mathcal R_\mathrm{AC}^* (\mu, \xi) = \frac12\|\xi\|_{L^2}^2
.
\label{eq:ac-grad-geometry}
\end{align} 
Cahn-Hilliard
is the gradient flow of
$F$
in unweighted $H^{-1}$, \ie
\begin{align}
    \mathcal R_\mathrm{CH} (\rho, \dot u) = \frac12\|\dot u\|_{H^{-1}}^2
,
\mathcal R_\mathrm{CH}^* (\rho, \xi) = \frac12\|\nabla\xi \|_{L^2}^2
.
\label{eq:ch-grad-geometry}
	\end{align}
    \label{ex:allen-cahn-hilliard}
\end{example}
\paragraph*{Geodesics and their Hamiltonian formulation.} For many considerations of gradient flows, the geodesic curves play an important role. These curves are obtained as minimizers of the length of all curves connecting two points:
\[
\gamma_{u_0\to u_1}\in\argmin_u \int_0^1 \frac12 \langle \bbG(u(s)) \dot u(s), \dot u(s) \rangle \dd s . 
\]
In the sense of classical mechanics, one may the consider the dissipation potential $\calR(u,\dot u)=\frac12\langle \bbG(u)\dot u, \dot u\rangle $ as a ``Lagrangian'' $L(u,\dot u)=\calR(u,\dot u)$ and the dual dissipation potential $\calR^*(u,\xi)= \frac12\langle \xi, \bbK(u)\xi\rangle $ as a Hamiltonian $H(u,\xi)=\calR^*(u,\xi)$. Then, minimizing the integral over $L$ is equivalent to solving the Hamiltonian system \EEE
\begin{align}
        \left\{
        \begin{aligned}
            \dot u &= \partial_\xi H (u, \xi)=  \partial_\xi \mathcal R^*(u, \xi) = \bbK(u) \xi,\\
            \dot \xi &= -\rmD_u H (u, \xi)= - \rmD_u \mathcal R^*(u, \xi)  ,
        \end{aligned}
        \right\}, \quad  u(0)=u_0, \ u(1)=u_1.
        \label{eq:intro-gfe-hamilton}
        \tag{H}
\end{align}
Here, the conditions for $u$ at $s=0$ and $s=1$ indicate that finding geodesic curves leads to solving a two-point boundary value problem. 

The theory for geodesics becomes particularly interesting in the case that $\calR^*$ is linear in the state $u$, because then $\rmD_u \calR^*(u,\xi)$ no longer depends on $u$. This means that the system \eqref{eq:intro-gfe-hamilton} decouples in the sense that the equation for $\xi$ no longer depends on $u$. This particular case occurs in the Wasserstein, Fisher-Rao, and consequently Wasserstein-Fisher-Rao space.
This structure allows for the derivation of the following characterizations of the geodesic curves and static formulations of the associated Riemannian distances.    
\EEE

\begin{example}
   [Wasserstein geodesics in Hamiltonian formulation]
\label{ex:Wasserstein geodesics}
For the Wasserstein case, the dual dissipation potential takes the form
\footnote{For ease of calculation, we always consider the $\frac{1}{2}$ scaling for quadratic dissipation potentials. That is, this case corresponds to the
geodesics of the $\frac12W_2^2$.
}
\begin{align*}
     H(\mu, \xi) 
     = \calR^*_{W_2}(\mu, \xi)
     =  \frac12 \| \nabla \xi \|^2_{L^2_\mu} 
     =\int \frac12 |\nabla \xi|^2 \rmd \mu.
\end{align*}   
The Onsager operator is given by $\bbK(\mu)\xi = - \DIV 
(\mu \nabla \xi)$ and the geodesic curves are characterized by
\begin{align}
        \left\{
        \begin{aligned}
            \dot \mu & = - \DIV \left(\mu \nabla \xi \right)
            ,\\
            \dot \xi & = {{-}}   \frac12| \nabla \xi | ^2 .
        \end{aligned}
        \right.
        \label{eq:intro-wgf-hamilton}
        \tag{Geod-W}
    \end{align}
Here, the first equation is the continuity equation which implies that $\mu$ is transported along the vector field $(t,x) \mapsto \nabla \xi(t,x)$, and the second equation is the Hamilton-Jacobi equation, which is notably independent of $\mu$.
The Hopf-Lax formula then gives the explicit characterization of the solution
$$
\xi(s,x) = \inf_{y} \big\{ \xi(0,y) + \frac1{2s} |x{-}y|^2 \big\},
$$
yielding the celebrated dual Kantorovich formulation of the Wasserstein distance.
See \cite{ambrosio2008gradient} for details. 
\end{example}

\begin{example}
   [Fisher-Rao (or Hellinger) geodesics in Hamiltonian formulation]
\label{ex:Fisher-Rao geodesics}
For the Fisher-Rao case in \eqref{eq:bb-formula-fr}, the primal and dual dissipation potential takes the form 
\begin{equation}
\begin{aligned}
        &
        \mathcal R_\FR(\mu,  \dot u)= \frac{1}2 \|\frac{\delta \dot u}{\delta {\mu}}\|^2_{L^2_{\mu}},
        \\
        &
         H(\mu, \xi) 
         = \calR^*_{\FR}(\mu, \xi)
         =  \frac12  \big\| \xi \big\|^2_{L^2_\mu} 
         =\int \frac12 \xi^2 \rmd \mu.
         \label{eq:fr-hamilton}
\end{aligned}
\end{equation}
The Onsager operator is given by $\bbK(\mu)\xi = \xi \mu $ and the geodesic curves are characterized by 
\begin{align}
        \left\{
        \begin{aligned}
            \dot \mu & = - \mu \xi ,\\
            \dot \xi & = {{-}} \frac12| \xi | ^2 .
        \end{aligned}
        \right.
        \label{eq:intro-FR-hamilton}
        \tag{Geod-FR}
\end{align}
Remarkably,
different from the Hamilton-Jacobi setting,
this system can be solved in the explicit form  
\[
\xi(s,x) = \frac{\xi(0,x)}{1{+}s \xi(0,{x})/2} \quad \text{and} 
\quad \mu(s,\rmd x) = \big(1{+}s \xi(0,x)/2\big)^2 \mu_0(\rmd x),
\]
where we already used the initial condition $\mu(0)=\mu_0$. 
Applying the final condition $\mu(1)=\mu_1$ we arrive at the explicit representation of the Fisher-Rao geodesic
\begin{align}
\omega(s) = \big( (1{-}s) \sqrt{\mu_0} + s \sqrt{\mu_1}\big)^2 
 = (1{-}s)^2 \mu_0 + 2s(1{-}s) \sqrt{\mu_0\mu_1} + s^2 \mu_1. 
\label{eq:geod-curve-FR}
\end{align}
See \citep{LasMie19GPCA} for details. 
Finally,
using
the explicit solution for $\xi(s,x)$ above,
one can show that the Fisher-Rao geodesic distance indeed admits the formula~\eqref{eq:fr-def}.
Formally, one can also obtain
a static dual Kantorovich type formulation
\begin{align*}
    \frac12
    \FR^2(\mu_0,\mu_1) =
        \sup_{(2 + \phi)(2 - \psi)=4}
        \left\{
            \int \psi\dd \mu_1
            -
            \int \phi\dd \mu_0 
        \right\}
        .
\end{align*}
\end{example}

\subsection{Reproducing kernel Hilbert space}
We first remember some basic facts about the reproducing kernel Hilbert spaces (RKHS), which are a class of Hilbert spaces that are widely used in approximation theory~\citep{wendland_scattered_2004,cucker_learning_2007} and machine learning~\citep{steinwart2008support}.

In this paper, we refer to 
a bi-variate function $k: \Omega \times \Omega \rightarrow \mathbb{R}$ as a symmetric positive definite kernel if $k$ is symmetric and, for all $n \in \mathbb{N}, \alpha_1, \ldots, \alpha_n \in \mathbb{R}$ and all $x_1, \ldots, x_n \in \Omega$, we have $\sum_{i=1}^n \sum_{j=1}^n \alpha_i \alpha_j k\left(x_j, x_i\right) \geq 0$.
If the inequality is strict, then $k$ is called strictly positive definite.
Here, the space $\Omega$ can be a subset of $\R^d$.
$k$ is a reproducing kernel if it satisfies the reproducing property, i.e., for all $x \in X$ and all functions in a Hilbert space $f \in \rkhs$, we have $f(x)=\langle f, k(\cdot, x)\rangle_{\rkhs}$.
Furthermore, the space $\rkhs$ is an RKHS if the Dirac map $\delta_x: \rkhs \mapsto \mathbb{R}, \delta_x(f):=f(x)$ is continuous.
It can be shown that there is a one-to-one correspondence between the RKHS $\rkhs$ and the reproducing kernel $k$.
The following fact regarding the RKHS, whose statement is adapted from \citep[Theorem~4.27]{steinwart2008support}, is instrumental to this paper.
\begin{theorem}    
[Integral operator]
Suppose the kernel is square-integrable $\|k\|_{L^2_{\rho}}^2:=\\\int k(x, x) d \rho(x)<\infty$ w.r.t. a probability measure $\rho$.
Then
the inclusion from the the associated RKHS \(\rkhs\) to $L^2_{\rho}$, \(\ID : \rkhs \rightarrow L^2_{\rho}\), is
continuous.
Moreover, its adjoint is the operator \(\Tkrho: L^2_{\rho} \rightarrow \rkhs\) defined by
$$
\Tkrho g(x):=\int k\left(x, x^{\prime}\right) g\left(x^{\prime}\right) d \rho\left(x^{\prime}\right), \quad g \in L^2_{\rho}
$$
$\Tkrho$ is Hilbert-Schmidt (\ie singular values are square-summable).
The integral operator
$$\K_\rho:= \ID \circ \Tkrho,
L^2(\rho) \to L^2(\rho)$$
is compact, positive, self-adjoint, and nuclear (\ie singular values are summable).
\label{thm:rkhs-int-operator}
\end{theorem}
We define the power of the integral operator $\K_\rho$ as, for $\alpha>0$,
$\K_\rho ^\alpha := 
\sum_{i=1}^\infty \lambda_i^\alpha \langle \cdot, \phi_i\rangle_{L^2_\rho} \phi_i$,
where $\lambda_i$ and $\phi_i$ are the eigenvalues and eigenfunctions of $\K_\rho$ given by the spectral theorem.
The image of the square root integral operator is the RKHS, i.e., 
$\rkhs  = \K_\rho ^\frac{1}{2} (L^2_{\rho})$ for some probability measure $\rho$.
See, e.g., \cite[Chapter 4]{cucker_learning_2007}.
Then,
for any $g\in\rkhs$, $\exists f\in L^2_{\rho}$ such that
    $g = \K ^{\frac12}_\rho f,
    \|g\|_\rkhs = \|f\|_{L^2_{\rho}}$.
Therefore, we can conveniently write down some formal relations
between the RKHS and $L^2$ norm
useful for our later analysis
\begin{align}
    \|g\|^2_\rkhs = \|\K_\rho ^{-\frac12} g\|^2_{L^2_{\rho}}=\iprod{g}{\K_\rho ^{-1} g}_{L^2_{\rho}},
    \quad
    \iprod{ g}{\Tkrho f}_{\rkhs} = \iprod{\ID g}{f}_{L^2_{\rho}},
    \label{eq:relation-rkhs-norm}
\end{align}
where $\K_\rho ^{-\frac12}$ denotes the inverse of $\K_\rho ^{\frac12}$.
The power of the kernel integral operator is prominently manifested in
the following nonparametric regression problem.
\begin{lemma}
[Kernel ridge regression estimator]
Given the target function $\xi\in L^2_{\rho}$,
the kernel ridge regression (KRR) problem for $\lambda>0$
\begin{align}
\inf_g 
\biggl\{
    \|g-\xi\|^2_{L^2_{\rho}} 
    +
    \lambda
    \|g\|_\rkhs^2 
\biggr\}
,
\end{align}
admits the closed-form solution
\begin{align}
g^* = (\K_\rho + \lambda\ID)^{-1} \K_\rho \xi.
\label{eq:opt-regression-sol}
\end{align}
\end{lemma}
To set the stage for our derivation later, we establish below
an alternative optimization formulation of the KRR solution.
\begin{lemma}
    [Alternative optimization problem of KRR estimation]
    The KRR solution~\eqref{eq:opt-regression-sol} coincides with the solution of the optimization problem
    \begin{align*}
        \inf_f
        \biggl\{
                \iprod{f -  \xi}{\K_\rho (f -  \xi)}_{L^2(\rho)}  + 
            \lambda \| f \|^2_{L^2(\rho)}
            \biggr\}
            .
    \end{align*}
    \label{lm:krr-alter}
\end{lemma}
One prominent applications of kernel methods to machine learning is the kernel maximum mean discrepancy (MMD) \citep{gretton2012kernel}, for measuring the discrepancy between probability measures.
It is a special case of the integral probability metric (IPM) family
that is defined as a weak norm
\begin{align}        
    \mmd( \mu, \nu) := \sup_{\|f\|_\rkhs \leq 1}\int f \dd (\mu-\nu)
    ,
    \label{eq:ipm-mmd-def-main}
\end{align}
where $\rkhs$ is the RKHS associated with the kernel $k$.
MMD is a metric on the space of probability measures if the kernel is positive definite.
Furthermore, its advantage lies in its simple structure as a Hilbert space norm
\begin{align}
    \mmd^2( \mu, \nu) = \|\K \left(\mu - \nu\right)\|_{\rkhs}^2
    =\int \int k(x, x') \dd (\mu - \nu)(x) \dd (\mu - \nu)(x')
    .
\end{align}
This form allows for efficient computation of the MMD via Monte Carlo sampling.
It can also be viewed as a form of interaction energy~\citep{ambrosio2008gradient} that dissipates in gradient-flow geometry, e.g., the Wasserstein flow of the MMD energy~\citep{arbel_maximum_2019}.
One of our contributions is to provide a precise relation between the MMD and the Fisher-Rao geometry from two different perspectives in Theorem~\ref{thm:mmd-fr-dynamic} and Proposition~\ref{prop:kernel-discrepancies}.

\section{Wasserstein space and gradient flows}
\label{sec:wasserstein}
We first visit the gradient flows in the Wasserstein type dissipation geometry.
First, we revisit the Stein gradient flow in Section~\ref{sec:w2stein},
where we establish the gradient flow structure and dissipation
potentials for Stein and its regularized version.
In Section~\ref{sec:dk-w2}, we derive a new Wasserstein-type gradient-flow geometry by drawing the parallel to the relation between Fisher-Rao and MMD.
We further show a related Cahn-Hilliard dissipation geometry.

\subsection{Gradient structure for the (regularized) Stein gradient flow}
\label{sec:w2stein}
The Stein geometry~\citep{duncan2019geometry}
has been studied in the context of sampling for statistical inference.
We first write down explicitly the gradient structure for the Stein gradient system by providing the dissipation potentials for the Stein geometry, implied in the dynamic formulation by \citet{duncan2019geometry}
\begin{align}
R_{\mathrm{Stein}}(\rho, u)
=\frac12\|  u \|_{\mathrm{Stein}, \rho}^2, 
\quad
\|  u \|_{\mathrm{Stein}, \rho}^2 
:=\inf \left\{  \|\K_\rho  v\|^2_\rkhs:  u=-\DIV (\rho\cdot \ \K_\rho v ) \right\}.
\label{eq:stein-norm}
\end{align}
We refer to $ \|  u \|_{\mathrm{Stein}, \rho}$ as the primal Stein norm
dual dissipation potential.
By Fenchel-duality, we find the velocity-kernelized dual dissipation potential of the Stein gradient flow
\begin{align}
R_{\mathrm{Stein}}^*(\rho, \xi)
=\frac12\iprod{\nabla \xi}{\K_\rho\nabla \xi}_{L^2_{\rho}}
=\frac12\| \K_\rho\nabla \xi \|^2_\rkhs
.
\end{align}
Using this gradient structure,
we obtain Stein (variational) gradient-flow equation
\begin{align}
    \partial_t \mu
    =
\DIV (\mu\cdot  \K_\mu\nabla  \dFdmu )
.
\end{align}

\paragraph*{Regularized Stein gradient flow as an approximation to Wasserstein gradient flow}
        Following a similar route as in the kernelized Fisher-Rao setting, we consider the regularized primal and dual dissipation potential
\begin{align}
    \mathcal R_{\operatorname{\Lstein}}(\rho, u) 
    &:=
    \mathcal R_{W_2} (\rho, u)
    +
    \lambda  R_{\mathrm{Stein}} (\rho, u)
    =
    \frac12\|u\|_{H^{-1}(\rho)}^2.
    + 
    \frac{\lambda}{2} \| u \|_{\mathrm{Stein}, \rho}^2
    \, \text{ for }u=\DIV(\rho v)
    \nonumber
    \\
    &=
    \frac12
    \left(
        \iprod{v}{v}_{L^2_{\rho}}
        +
        \iprod{v}{\K_\rho^{-1}v}_{L^2_{\rho}}
    \right)
    =
    \frac12
    \iprod{v}{\K_\rho^{-1}(\K_\rho+ \lambda \ID)v }_{L^2_{\rho}},
    \label{eq:stein-reg-dissipation-primal}
    \\
    \mathcal R_\mathrm{\Lstein}^*(\rho, \xi) &=
    \frac{1}{2}
    \iprod{\nabla \xi}
    {
        \left(
            \K_\rho + \lambda \ID
            \right)
            ^{-1}
            \K_\rho
        \nabla\xi
        }_{L^2_{\rho}}
    ,
    \label{eq:stein-reg-dissipation}
\end{align}
resulting in the 
following approximate Wasserstein gradient system.
\begin{proposition}
    [Kernel-approximate Wasserstein gradient flow with KRR velocity field]
The generalized gradient system $(\mathcal P,F, \Lstein)$ generates the gradient flow equation
where the approximate growth field $r$ is given by the KRR solution 
\begin{align}
    \label{eq:KRR-stein}
        \dot {\mu_t} = \DIV ({\mu_t}\cdot v_t),
        \quad 
    v_t = \argmin_f
    \biggl\{
            \|f - \nabla \dFdmut\|^2_{L^2_{\mu_t}}  + 
        \lambda \| f \|^2_{\rkhs}
        \biggr\}
        .
\end{align}
Specifically, the closed-form solution for the velocity field is
\begin{align}
    v_t = (\K_\mu +\lambda\ID)^{-1}\K_\mu \nabla  \dFdmu
    .
    \label{eq:stein-reg-gfe}
\end{align}
\end{proposition}
It is worth noting that the above approximation fits the setting of learning the diffusion model~\eqref{eq:reverse-sde-diffusion-model}, where the velocity is approximated using KRR.
Later, we rigorously justify this approximation.

\subsection{De-Stein: de-kernelized Wasserstein geometry}
\label{sec:dk-w2}
As we have witnessed in Section~\ref{sec:mmd-reg}, 
the MMD bears the intuition of the \emph{de-kernelized and flat Fisher-Rao distance}.
What then is the de-kernelized (flat) Wasserstein geometry, as MMD is to Fisher-Rao?
Our starting point is the following de-kernelized $H^{-1}$-type norm
\begin{align*}
    \|  u \|_{\textrm{De-Stein}}^2 
    :=
    \inf \left\{  \|  v\|^2_\rkhs:  
    u=-\DIV (\rho\cdot \ \K_\rho^{-1} v ) \right\}
    =
    \inf \left\{  \|  v\|^2_\rkhs:  
    u=-\DIV (\cdot \ \K^{-1} v ) \right\}
    .
\end{align*}
Similar to the gradient structure of MMD,
this quantity no longer depends on the measure $\rho$.
As a consequence of the above formulation, the primal and dual dissipation potential for the approximation system are \emph{state-independent}
\begin{align*}
    R_{\textrm{De-Stein}}(u)
    =\frac12\|  u \|_{\textrm{De-Stein}}^2, 
    \quad
    R^*_{\textrm{De-Stein}}(\xi)
    =\frac12\|\K^{-\frac12} \nabla \xi \|_{L^2}^2
    =\frac12\| \nabla \xi \|_{\rkhs}^2
    .
\end{align*}
Therefore, like the MMD, this geometry is a flat geometry and its gradient flow equation is
\begin{align*}
    \partial_t \mu=-\DIV ( \ \K^{-1}\nabla \dFdmu ) 
    .
\end{align*}
We can now write down the dynamic formulation of a new distance, which we term the 
\emph{\underline{de}-kernelized Wasser\underline{stein}} (De-Stein) distance\footnote{
Due to the presence of static dual representation in Proposition~\ref{thm:sobolev-ipm},
this distance should be more appropriately termed
kernel-Sobolev distance.
However, similar terms have already been used in the machine learning literature to denote a heuristically regularized $H^{-1}$ geometry.
}.
\begin{equation}
    \label{eq:bb-formula-dekern-w2}
    \begin{aligned}
        \operatorname{De-Stein}^2(\mu,\nu)
        =
        \inf 
        \left\{\int_0^1
        \|  \nabla \xi_t \|^2_{\rkhs}
        \dd t
        \middle|
         \partial_t u=-\DIV ( \K^{-1}\nabla \xi_t )
         ,
         u(0) =  \mu,
        u(1)=  \nu\right\}
        .
    \end{aligned}
\end{equation}
The Hamiltonian formulation for De-Stein is
\begin{equation*}
    \left\{
    \begin{aligned}
        \dot \mu &= -\DIV ( \K^{-1}\nabla \xi) ,\\
        \dot \xi &= 0.
    \end{aligned}
    \right.
\end{equation*}
where the Hamilton-Jacobi equation in the Wasserstein geometry is replaced by the ``static'' adjoint variable $\xi$ due to the state-independent dissipation potential~\eqref{eq:bb-formula-dekern-w2}.
More concretely,
following the derivation of the relation between the 
static Kantorovich dual formulation and dynamic Benamou–Brenier formulation (see, \eg \citep{otto2000generalization}),
we now derive the static defition of the metric.
Different from the Wasserstein and Stein setting, we only need one static (\ie time-independent) test function,
because of the adjoint (geodesic) equation in the Hamiltonian dynamics $\partial_t \zeta = 0$.
Consequently, we obtain a simple static formulation as a weak norm, similar to an IPM.
\begin{proposition}
    [Static dual formulation of the De-Stein distance]
The dynamic formulation of the De-Stein distance~\eqref{eq:bb-formula-dekern-w2} is equivalent to the static dual formulation
        \begin{align}
            \operatorname{De-Stein}^2(\mu,\nu)
            =    \sup_\zeta
            \left\{
                \int \zeta\dd (\mu - \nu)
                -\frac14
                 \|\nabla \zeta\|^2_{\rkhs}
            \right\}
            .
            \label{eq:sobolev-rkhs-static-dual}
        \end{align}
        \label{thm:sobolev-ipm}
\end{proposition}
The intuition is self-evident --- compared with the static dual of MMD, the regularization by the RKHS norm $\|\zeta\|_\rkhs$ is replaced by the RKHS norm of its gradient.
The Euler-Lagrange equation of the optimization problem in \eqref{eq:sobolev-rkhs-static-dual} is
$\frac12\DIV\left(- \K^{-1}\nabla \zeta\right) = \mu-\nu$.
Plugging it back into \eqref{eq:sobolev-rkhs-static-dual} and integrating by parts, 
we find
\begin{align*}
    \operatorname{De-Stein}^2(\mu,\nu)
    =
    {\inf _v }
    \left\{     \|v\|^2_{\rkhs}
    \,\ST\,
    \DIV\left(- \K^{-1}v\right) = 2(\mu-\nu)
    \right\}
    ,
\end{align*}
which is simply a kernel-weighted $H^{-1}$ norm.

\subsection{Flattened Wasserstein, Cahn-Hilliard, and Sobolev discrepancy}
\label{sec:lin-w}
Mirroring the development of the flattened Fisher-Rao setting in Section~\ref{sec:lin-fr},
we now focus on a similar construction in the Wasserstein setting.
Similar to the De-Stein setting,
we obtain the state-independent dissipation potentials
\begin{align*}
    R_{\overline{W}_2}(u)
    =\frac12\|  u \|_{H^{-1}_{\omega}}^2, 
    \quad
    R^*_{\overline{W}_2}(\xi)
    =\frac12\| \nabla \xi \|_{L_{\omega}^2}^2
    ,
\end{align*}
where the reference measure $\omega$ is fixed.
Its dynamic formulation is given by
\begin{equation}
    \label{eq:bb-formula-w2-flattened}
    \begin{aligned}
        \overline{W}_{\omega}^2(\mu,\nu)
        =
        \min 
        \left\{\int_0^1
        \|  \nabla \xi_t \|^2_{L^2_\omega}
        \dd t
        \middle|
         \partial_t u=-\DIV ( \omega\cdot \nabla \xi_t )
         ,
         u(0) =  \mu,
        u(1)=  \nu\right\}
        .
    \end{aligned}
\end{equation}
The adjoint equation $\partial \zeta_t = 0$ implies the static dual formulation
\begin{align}
    \overline{W}_{\omega}^2(\mu,\nu)
    =    \sup_\zeta
    \left\{
        \int \zeta\dd (\mu - \nu)
        -\frac14
         \|\nabla \zeta\|^2_{L^2_\omega}
    \right\}
    .
    \label{eq:w2-flattened-static-dual}
\end{align}
Similar to the setting after Proposition~\ref{thm:sobolev-ipm},
we find
the static dual formulation is equivalent to
\begin{align*}
    \overline{W}_{\omega}^2(\mu,\nu)
    =
    \inf_\xi
    \left\{     \|\xi\|^2_{L^2_\omega}
    \,\ST\,
    \frac12
    \DIV (\xi \cdot \omega)  = \mu-\nu\right\}
    ,
\end{align*}
which coincides with the weighted $H^{-1}_{\omega}$ norm.
If the reference measure $\omega$ is chosen as the Lebesgue measure,
then $\frac12 \Delta \zeta =  \left(\mu-\nu\right)$.
Consequently, we obtain the static formulation
\begin{align}
    \overline{W}^2(\mu,\nu)
    =
    \|\mu - \nu \|^2_{H^{-1}}
    =
    -
    \int \left(\mu-\nu\right)\Delta^{-1} \left(\mu-\nu\right)
    \dd x
    ,
\end{align}
which is equivalent to
the
classical Cahn-Hilliard $H^{-1}$ Hilbert space
in Example~\ref{ex:allen-cahn-hilliard}.
If the reference measure $\omega$ is chosen as $\omega=\nu$,
the flattened Wasserstein distance is 
the $H^{-1}_\nu$ norm,
which is equivalent to the Sobolev discrepancy proposed by \citet{mrouehSobolevDescent2019}.

\section{Fisher-Rao space and gradient flows}
\label{sec:fisher-rao}
This section addresses the main subject of this paper, the Fisher-Rao-type gradient flow geometry.
We first study the kernelized Fisher-Rao geometry in Section~\ref{sec:kern-FR} and provide the principled gradient structure of the resulting kernelized Fisher-Rao gradient flow. Its growth field is an approximation to that of the Fisher-Rao.
In Section~\ref{sec:mmd-reg}, we perform the inverse operation to
``de-kernelize'' the Fisher-Rao geometry.
Consequently, we obtain a flat (in the sense of Riemannian manifold) gradient-flow geometry, which we show is equivalent to the MMD.
Due to the technical nature of the gradient flow theory, we refer to

\subsection{Kernel approximation of Fisher-Rao gradient flows}
\label{sec:kern-FR}
In the machine learning literature, the term \emph{kernelization} has been used
in many contexts. 
We first make precise what kernelization entails in this paper through the following operation on the dual dissipation potentials of gradient systems.
\begin{definition}
[Kernelization of gradient systems]
    \label{def:f-v-kernelize}
Given a gradient system defined by $(X, F, \mathcal R^*)$, where $\mathcal R^*$ is the dual dissipation potential, we say its force-kernelization counterpart is $(X, F, \mathcal R^*_{F-k})$, where
the force-kernelized dual dissipation potential is defined by 
\begin{align}
    \mathcal R^*_{F-k}(u,  \xi) :=\mathcal R(u, \K_u^\frac12\, \xi)
    .
\end{align}
If the original $\mathcal R^*$ depends on the generalized force $\xi$ only through its gradient $\nabla \xi$, denoted by  $\widetilde{\calR^*}(\nabla \xi) = \mathcal R^*(\xi)$. Then, its velocity-kernelization is $(X, F, \mathcal R^*_{V-k})$,
\begin{align}
    \mathcal R^*_{V-k}(u, \xi):= \widetilde{\calR^*}(u, \K_u^\frac12\, \nabla\xi)
    .
\end{align}
\end{definition}
Equivalently for the Fisher-Rao and Wasserstein type flows,
we can define kernelization using the operator $\calT_{k,u}$ in Theorem~\ref{thm:rkhs-int-operator} as a change the dissipation potentials
\begin{align*}
    \text{Fisher-Rao: }
    \calR^*(\rho, \xi) = \frac12\iprod{\xi}{\xi}_{L^2_\rho} 
    &\longrightarrow \calR_k^*(\rho, \xi) = \frac12\iprod{\Tkrho\xi}{\Tkrho\xi}_{\rkhs} , \\
    \text{Wasserstein: }
    \calR^*(\rho, \xi) = \frac12\iprod{\nabla \xi}{\nabla \xi}_{L^2_\rho} 
    &\longrightarrow \calR_k^*(\rho, \xi) = \frac12\iprod{\Tkrho\nabla \xi}{\Tkrho\nabla \xi}_{\rkhs} .
\end{align*}
This kernelization relation can also be stated using the Riemannian tensor $\bbG$ and the Onsager operator $\bbK$; see Theorem~\ref{thm:kernelization-riemann}.
Note that velocity-kernelization of Wasserstein gradient flow, under the name of Stein geometry, has already been investigated, \eg  by \cite{duncan2019geometry}.
As overviewed in Section~\ref{sec:intro},
we will construct systems in the kernelization relation illustrated in Figure~\ref{fig:kernelization-relations-FR}.

Using the above definition for kernelization,
we now construct a new geometry by force-kernelizing the Fisher-Rao gradient system.
\begin{definition}
[Dynamic formulation of kernelized Fisher-Rao distance]
The kernelized Fisher-Rao distance is defined by the following dynamic formulation
\begin{equation}
    \label{eq:bb-formula-fr-kernel}
        \begin{aligned}
            \FRk^2(\mu,\nu)
    =
        \min_{\mu, \xi_t} 
        \left\{\int_0^1
        \| {{\K_\mu \xi_t}}\|^2_{\rkhs}
        \dd t
        \ \
        \middle \vert \ 
        \dot \mu = -  \mu {\K_\mu  \xi_t},
        \
       \mu(0) = \mu,
        \
        \mu(1)= \nu\right\}
        .
        \end{aligned}
    \end{equation}
\end{definition}
From the Hamiltonian perspective in \eqref{eq:intro-gfe-hamilton}, we can derive the
geodesic equation
\begin{align}
    \left\{
    \begin{aligned}
        \dot \mu & = - \mu \xi ,\\
        \dot \xi & = {{-}} \xi \cdot \K_{\mu_t} \xi .
    \end{aligned}
    \right.
    \label{eq:adjoint-fr-kernel}
    \tag{Geod-FR-k}
\end{align}
However, is important to note that the geodesic equation above is only the necessary condition for optimality.
That is, we have not proved whether its solution exists, in contrast to the Hamiltonian system \eqref{eq:intro-FR-hamilton}.
This is due to the coupling introduced by the state-dependent integral operator $\K_{\mu_t}$.
Furthermore, as a technical point, the integral operator used here is defined w.r.t a non-negative measure rather than a probability measure; see, \eg \citep{conway1985course}.
In the gradient structure of the kernelized Fisher-Rao gradient system,
the corresponding primal and dual dissipation potentials are
\begin{align}
\mathcal R_{\FRk}(\rho, u)
= \frac12 \|  \frac{\delta u}{\delta \rho}\|^2_\rkhs
,
\quad
\mathcal  R^*_{\FRk}(\rho, \xi)
= \frac12\iprod{\xi}{\K_\rho  \xi}_{L^2_{\rho}}
= \frac12 \|\K_\rho  \xi\|^2_\rkhs
.
\label{eq:kernel-fr-prima-dual-dissip}
\end{align}
Therefore,
the gradient-flow equation of the $\FRk$ gradient system
$(\Mplus, F, \FRk)$
is the reaction equation \emph{kernelized growth field}
\begin{align}
    \dot \mu = - \mu \K_\mu  \dFdmu
    .
\end{align}
Going beyond kernelization,
we derive
the approximation to the original Fisher-Rao dynamics by constructing the following regularized dissipation geometry,
\ie
adding the kernelized Fisher-Rao dissipation potential $\mathcal R_{\FRk}$~\eqref{eq:kernel-fr-prima-dual-dissip}
to that of the pure Fisher-Rao $\mathcal R_\FR$
\eqref{eq:fr-hamilton}
\begin{align}
        \calR_{\LFRk} (\rho,  \dot u) &:=
        \mathcal R_\FR + \lambda\cdot \mathcal R_{\FRk}
        =
        \frac12 \big\|\frac{\delta \dot u}{\delta {\rho}}\big\|^2_{L^2_{\rho}}
        +
        \frac{\lambda}2 \big\|  \frac{\delta \dot u}{\delta \rho} \big\|_\rkhs^2
        =
        \frac{1}{2}
        \iprod{\frac{\delta \dot u}{\delta \rho}}
        {
            \K_\rho^{-1}\left(
                \K_\rho {+} \lambda \ID
            \right)
            \frac{\delta \dot u}{\delta \rho}
            }_{L^2_{\rho}}
        ,
        \nonumber
        \\
        \calR_{\LFRk}^*(\rho, \xi) &=
        \frac{1}{2}
        \iprod{\xi}
        {
            \left(
                \K_\rho {+} \lambda \ID
                \right)
                ^{-1}
                \K_\rho
            \xi
            }_{L^2_{\rho}}
        .
        \label{eq:fr-kernel-reg-dual}
\end{align}
    Then, we obtain the following
        approximate gradient-flow equation
        with an approximate growth field given by the kernel ridge regression solution 
        $(\K_\rho + \lambda\ID)^{-1}  \K_\rho \dFdmu$.
\begin{proposition}
        [Kernel-approximate Fisher-Rao gradient flow]
    The generalized gradient system $(\Mplus,F, \calR_{\LFRk})$ generates the gradient flow equation where the approximate growth field $r$ is given by the KRR solution 
    \begin{align}
    \dot {\mu_t} = - {\mu_t}\cdot r_t , \quad
        r_t = \argmin_f
        \biggl\{
                \left\|f -  \dFdmut\right\|^2_{L^2_{\mu_t}}  + 
            \lambda \| f \|^2_{\rkhs}
            \biggr\}
            .
            \label{eq:gfe-krr-FRk}
    \end{align}
    Specifically, the closed-form solution for the growth field is
    \begin{align}
        \dot \mu = - \mu\cdot (\K_\mu +\lambda\ID)^{-1}\K_\mu  \dFdmu
        .
        \label{eq:fr-krr-reaction}
    \end{align}
    \end{proposition}
    \begin{remark}[Relation between KRR and infimal convolution]
    \label{rem:KRR.vs.Infimal}
        The above 
        optimization problem~\eqref{eq:gfe-krr-FRk}
        is not equivalent to the infimal convolution of $\calR_\LFRk$ and $\calR^*_{\LFRk}$ defined in \eqref{eq:fr-kernel-reg-dual}. 
        However, it is easy to check using Lemma~\ref{lm:krr-alter} 
        that
    \[
    r_t=\argmin_g \left\{ \big\| g \big\|_{L^2_\mu}^2 + 
    \frac1\lambda  \big\|  \calK_\mu^{1/2} (g{-}\xi) 
    \big\|^2_{ L^2_\mu} \right\}
    .
    \]
    This quadratic form matches the definition of $\calR^*_{\LFRk}$~\eqref{eq:fr-kernel-reg-dual} as an infimal convolution,
    \ie $\calR^*_{\LFRk}(\mu, \xi) = \inf_{g} \left\{ \frac12\big\| g \big\|_{L^2_\mu}^2 +
    \frac1{2\lambda}  \big\|  \calK_\mu^{1/2} (g{-}\xi)
    \big\|^2_{ L^2_\mu} \right\}$.
    \end{remark}

\subsection{MMD as de-kernelized Fisher-Rao distance}
\label{sec:mmd-reg}
We now take a different direction from the previous subsection to de-kernelize the Fisher-Rao geometry~\eqref{eq:bb-formula-fr}.
The result is
a previously unknown connection: the resulting geometry is equivalent to the MMD geometry.

We now develop the gradient flow structure of the MMD\footnote{We follow the naming convention of the Wasserstein gradient flow and refer to the gradient flows in the MMD dissipation geometry as the MMD gradient flow. This is distinct from the setting considered in \citep{arbel_maximum_2019}, which is a WGF with the MMD energy.}.
The
primal and dual dissipation potentials are trivial due to the flatness of the MMD geometry
\begin{align}
\mathcal{R}_{\mmd} (  u) = \frac12 \| \K_\rho \frac{\delta\mu}{\delta \rho}  \|_\rkhs^2
=
\frac12 
\| \K {\mu}  \|_\rkhs^2,
\quad
\mathcal{R}_{\mmd}^* (  \xi) 
= \frac12 \iprod{\xi}{\K_\rho^{-1}\xi}_{L^2_{\rho}}
= \frac12 \iprod{\xi}{\K^{-1}\xi}_{L^2}
.
\label{eq:mmd-prima-dual-dissip}
\end{align}
It is important to note that the MMD dissipation potentials are \emph{state-independent}, i.e., they are not functions of the measure $\rho$
since
$\rho\cdot \K_\rho^{-1} v = \K^{-1} v$.
\begin{proposition}
[Gradient flow equation in the MMD geometry]
The gradient-flow equation in the MMD geometry is given by
\begin{align}
\dot \mu
= -\K ^{-1} \dFdmu
.
        \label{eq:mmd-gfe-primal}
    \end{align}
\label{thm:mmd-gf}
\end{proposition}
Proposition~\ref{thm:mmd-gf} gives the intuition that the MMD gradient-flow equation is equivalent to
a \emph{reaction equation with the de-kernelized growth field}
$\dot \mu= -\mu \K_\mu^{-1} \dFdmu$.
It is worth noting that the gradient-flow equation can be stated in the dual space,
$\frac{\dd}{\dd t}\K  {\mu} =- \dFdmu$,
where $\K  {\mu}$ is the kernel-mean embedding~\citep{smolaHilbertSpaceEmbedding2007}.

We now derive the main result of this section
using the dynamic formulation
\begin{equation}
    \label{eq:bb-formula-mmd}
    \begin{aligned}
        \mmd^2(\mu,\nu)
        =
        \min
\left\{        \int_0^1
        \|  \xi_t \|^2_{\rkhs}
        \dd t
        \
        \middle \vert \ 
         \dot u = - \K^{-1}   \xi_t,
          u(0) =  \mu,
         u(1)=  \nu\right\}
        .
    \end{aligned}
\end{equation}
Because of its flat structure,
the adjoint equation for the MMD is simply $\dot \xi = 0$.
It is also easy to verify the following lemma.
\begin{lemma}
    [Unconstrained dual formulation of squared-MMD]
    \label{thm:mmd-variational-form}
    The squared-MMD admits the unconstrained dual representation
    \begin{align}
    \operatorname{MMD}^2( \mu, \nu)
    =
    \sup_{h\in\rkhs}
    \intprod{h}{ (\mu- \nu)}
    -\frac{1}{4}\|h\|^2_\rkhs
                .
                \label{eq:dual-mmd-unconstrained}
    \end{align}
    \end{lemma}
The MMD geodesic curve is simply the straight line between $\mu$ and $\nu$, 
$u(t) = (1-t)\mu + t\nu$.
Therefore, when both $\mu$ and $\nu$ are probability measures,
the solution along the MMD geodesic remains a probability measure.
This is significantly simplified compared to the Fisher-Rao setting as remarked in footnote~\ref{footnote:spherical-FR-normalize}.
To be consistent with the FR setting, we also consider the MMD between non-negative measures instead of only probability measures.

Summarizing,
we present our main result regarding the MMD-Fisher-Rao relation.
\begin{theorem}
    The dynamic formulation of the force-kernelized (Definition~\ref{def:f-v-kernelize}) squared MMD~\eqref{eq:bb-formula-mmd} coincides with that of the squared Fisher-Rao distance~\eqref{eq:bb-formula-fr}.
    \label{thm:mmd-fr-dynamic}
\end{theorem}

\paragraph{Riemannian metric perspective}
Using the perspective in Section~\ref{sec:grad-flow-geod},
we show another perspective of the kernelization of Fisher-Rao and MMD following Definition~\ref{def:f-v-kernelize}.
Using the dissipation geometry~\eqref{eq:fr-hamilton},
one can easily show that
the Fisher-Rao
Riemannian tensor (see, \eg \cite{chenGradientFlowsSampling2023}) and Onsager operator are
$$
 \bbG_\FR(\nu) = \frac1\nu\cdot , \quad\bbK_\FR(\nu) = \nu\cdot .
$$
Using the RKHS-$L^2$ relation~\eqref{eq:relation-rkhs-norm},
the state-independent counterparts for the MMD are
$$
\bbG_{\mmd}= \K , \quad\bbK_{\mmd}= \K^{-1}.
$$
Therefore, following Theorem~\ref{thm:mmd-fr-dynamic}, we conclude:
\begin{theorem}
    [Kernelization of Fisher-Rao Riemannian tensor]
    \label{thm:kernelization-riemann}
    The
    MMD and\\Fisher-Rao Riemannian tensors and Onsager operators are related by the integral operator $\K_\nu$:
    \begin{align*}
        \bbG_{\mmd} = \K_\nu\circ\bbG_\FR (\nu), \quad \bbK_{\mmd} = \bbK_\FR (\nu) \circ \K_\nu^{-1}
        .
    \end{align*}
\end{theorem}
One implication can be seen in the following
Fisher-Rao minimizing movement
$$
\min_{\mu\in \Mplus} F(\mu) + \frac1{2\tau }\langle{\mu-\nu}, \bbG_\FR(\nu) {(\mu-\nu)} \rangle_{L^2},
$$
Here $\mu$ is a given non-negative measure.
In such cases, kernelization can be used to construct the MMD minimizing movement by replacing the Fisher-Rao Riemannian tensor:
$$
\min_{\mu\in \Mplus} F(\mu) + \frac1{2\tau }\langle{\mu-\nu}, \bbG_{\mmd} {(\mu-\nu)} \rangle_{L^2}.
$$

\paragraph*{Spherical MMD gradient flow}
In machine learning applications,
one often considers gradient flows but is restricted to the probability space.
It is important to note that the aforementioned gradient flows in the MMD space is defined on the non-negative measures $\Mplus$.
For this reason,
\citet{gladin2024interaction} proposed a spherical MMD gradient flow
by restricting the MMD gradient flow to the probability space.
For example, we are motivated by the gradient flow whose discretization can be simulated by the MMD-JKO scheme over the probability space:
\begin{align}
    \label{eq:mmd-mms}
    \mu^{k+1}
    \gets\argmin_{\mu\in\cal P} F(\mu ) + \frac1{2\eta}{\mmd}^2(\mu, \mu^{k})
    .
\end{align}
\begin{proposition}
    [Spherical MMD gradient flow equations \citep{gladin2024interaction}]
    The spherical MMD gradient flow equation is 
    given by (where $1$ denotes the constant function)
    \begin{align}
        \label{eq:spherical-MMD-gfe}
        \dot \mu 
    = - \K^{-1}\left(\dFdmu - \frac{\int \K^{-1}\dFdmu}{\int \K^{-1}1}\right)
        .
    \end{align}
    The gradient flow is mass-preserving, \ie $\int \dot \mu = 0$.
    \label{prop:spherical-mmd-gfe}
\end{proposition}
The variational problem~\eqref{eq:mmd-mms} has been used by \citet{zhu_kernel_2021,kremer2023estimation,gladin2024interaction} for machine learning applications to take advantage of MMD's favorable properties for practical optimization.

\subsection{Fisher-Rao-Regularized MMD gradient flow}
In the MMD gradient flow equation~\eqref{eq:mmd-gfe-primal},
the operator $\K$ maybe compact.
    We now consider the regularization for the MMD dissipation potential
    by simply adding the Fisher-Rao potential
    \begin{align}
        \RfrMMD (\rho,  u) :=
        \mathcal R_{\mmd} + \lambda\mathcal R_{\FR}
        =
        \frac1{2}
        \left\| \K_\rho \frac{\delta\mu}{\delta \rho}  \right\|_\rkhs^2
        + \frac{\lambda}2 \left\| \frac{\delta u}{\delta {\rho}}\right\|^2_{L^2_{\rho}}
        =
        \frac{1}{2}
        \iprod{\frac{\delta u}{\delta \rho}}
        {\left(
            \K_\rho {+} \lambda \ID
            \right)
            \frac{\delta u}{\delta \rho}
            }_{L^2_{\rho}}
        .
        \label{eq:fr-mmd-primal-dissip-2}
    \end{align}
    and, consequently, its dual potential is obtained by the Legendre transform
$        \RfrMMD^*(\rho, \xi) =
        \frac{1}{2}
        \iprod{\xi}
        {
            \left(
                \K_\rho + \lambda \ID
                \right)
                ^{-1}
            \xi
            }_{L^2_{\rho}}$.
Therefore,
    its gradient flow equation is now equipped with a
    regularized deconvolution growth field
    \begin{align}
        \dot \mu = - \mu\cdot (\K_\mu +\lambda\ID)^{-1}  \dFdmu
        .
        \label{eq:fr-krr-reaction-gfe-1}
    \end{align}
An alternative intuition can be seen in the following kernel-mean-embedded gradient-flow equation.
    \begin{proposition}
        [Fisher-Rao-Regularized MMD gradient flow in the embedding space]
    The generalized gradient system $(\Mplus,F, \mmd_\mathrm{FR})$ generates the gradient flow equation
    \begin{align}
        \frac{\dd}{\dd t} (\K \mu) = -  r,
    \end{align}
    where $r$ is the Fisher-Rao growth field~\eqref{eq:gfe-krr-FRk}, i.e.,
    \begin{align}
        r = \argmin_f
        \biggl\{
                \|f -  \dFdmu\|^2_{L^2_\mu}  + 
            \lambda \| f \|^2_{\rkhs}
            \biggr\}
            .
            \label{eq:gfe-krr-FRk-2}
    \end{align}
    \label{thm:fr-mmd-gf}
    \end{proposition}

With the above regularized MMD geometry,
we now construct the regularized MMD geodesic distance.
\begin{definition}
    [Regularized MMD geodesics]
    The Fisher-Rao-regularized MMD is defined by the dynamic formulation
    \begin{equation}
            \begin{aligned}
                \mmd_\mathrm{FR}^2(\mu, \nu)
            :=
            \min_{u, \xi_t} \ \  & 
            \int_0^1
            \iprod{\xi_t}
            {
                \left(
                    \K_{\mu_t} + \lambda \ID
                    \right)
                    ^{-1}
                \xi_t
                }_{L^2_{{\mu_t}}}
            \dd t
            \\
            \operatorname{s.t.} \ \  & 
            \dot \mu = - \mu\cdot (\K_\mu +\lambda\ID)^{-1}   \xi_t
    ,
            u(0) = \mu,
            u(1)= \nu
            .
            \label{eq:bb-formula-mmd-fr}
            \end{aligned}
        \end{equation}
\end{definition}
It turns out the geodesic distance admits a static formulation that is simply a regularized MMD metric by adding the squared Fisher-Rao distance.
\begin{proposition}
        [Equivalence between static and dynamic formulation of $\mmd_\mathrm{FR}$]
        The dynamic formulation~\eqref{eq:bb-formula-mmd-fr}
        coincides with the static formulation of the \emph{Fisher-Rao-regularized MMD metric}
           \begin{align}
               \label{eq:mmd-plus-chisqr}
               \mmd_\mathrm{FR}^2(\mu,\nu) 
               :=
               \mmd^2 (\mu,\nu) + \lambda\cdot 
               \FR^2(\mu, \nu)
               .
           \end{align}
    \end{proposition}
Formally, as $\lambda\to 0$, the regularized MMD metric~\eqref{eq:mmd-plus-chisqr} trends toward the original MMD metric.
Alternatively, one may also trivially consider 
$\lambda\cdot\mathcal R_{\mmd} +  \mathcal R_{\FR}$ in \eqref{eq:fr-mmd-primal-dissip-2} and
$\lambda\cdot\mmd^2 (\mu,\nu) + \FR^2(\mu, \nu)$, \ie using the MMD as a regularization for the Fisher-Rao geometry.
Consequently, we obtain a gradient flow that approximates (as $\lambda\to 0$) the $\FR$ gradient flow with the gradient-flow equation
    $\dot \mu = - \mu\cdot (\lambda \K_\mu +\ID)^{-1}  \dFdmu$.
\EEE

\subsection{Flattened Fisher-Rao, Allen-Cahn, and $\varphi$-divergences}
\label{sec:lin-fr}
Motivated by the relation between the MMD and the Fisher-Rao distance,
we now discuss another class of divergences via an analogous construction from the dynamic formulation.
This
amounts to changing the 
state-dependent
dissipation potential \eqref{eq:fr-prima-dual-dissip}
to the
state-independent
\begin{align}
\mathcal R_{\overline{\FR}}(  u) = \frac{1}2 \|\frac{\delta u}{\delta \omega}\|^2_{L^2_\omega},
\
\mathcal R_{\overline{\FR}}^*(\xi) =\frac1{2} \| \xi \|^2_{L^2_\omega},
    \label{eq:fr-prima-dual-dissip-flattened}
\end{align}
for a fixed reference measure $\omega$.
Like the MMD and the classical Allen-Cahn (i.e.\,$L^2$),
the flatten disspation potentials are state-independent.
Using the dissipation potential~\eqref{eq:fr-prima-dual-dissip-flattened},
We obtain the dynamic formulation
\begin{equation}
    \label{eq:bb-formula-fr-flattened}
        \begin{aligned}
        \overline{\FR}^2_\omega(\mu,\nu)
        =
        \min_{u, \xi} 
        \left\{\int_0^1
        \| \xi_t\|^2_{L^2_\omega}
        \dd t
        \ \
        \middle \vert \ \  
        \dot u =  \omega\cdot \xi_t,
        \
        u(0) = \mu,
        \
        u(1)= \nu\right\}
        .
        \end{aligned}
    \end{equation}
Similar to the MMD setting,
the adjoint equation of the Hamiltonian dynamics
simplifies to $\dot \zeta_t = 0,$
resulting in the following static formulation in Proposition~\ref{thm:fr-linearized-static}.
The proof is omitted since it is
a slight modification of the proof in
\citep[Section~3]{otto2000generalization}
that of the de-kernelized Wasserstein distance
we show later.
Similar to the linearized optimal transport~\citep{wang2013linear} and generalized geodesics~\citep{ambrosio2008gradient}, it is natural to consider a reference measure $\omega$ along the
\emph{Fisher-Rao geodesic} between $\mu$ and $\nu$.
By doing so, we recover the following connections with the $\varphi$-divergences.
\begin{proposition}
    [Static formulation of flattened Fisher-Rao distance]
    The flattened Fisher-Rao distance~\eqref{eq:bb-formula-fr-flattened} 
    is equivalent to the static formulation
    \begin{align}
        \overline{\FR}^2_{\omega}(\mu,\nu)
        =
        \sup_{\zeta}
        \left\{
            \int \zeta\dd (\mu - \nu)
            -\frac14
            \|\zeta\|^2_{L^2_{\omega}}
        \right\}
        .
        \label{eq:fr-linearized-static}
    \end{align}
    If the reference measure is the Lebesgue measure $\omega=\Lambda$,
    then the flattened Fisher-Rao 
    distance coincides with the $L^2$ norm $\overline{\FR}^2_{\omega}(\mu,\nu)= \|\mu - \nu\|^2_{L^2}$.
    The resulting gradient flow is the $L^2$ Hilbert space gradient flow (classical Allen-Cahn, Example~\ref{ex:allen-cahn-hilliard}).

    Furthermore,
    Suppose the reference measure $\omega$ is chosen along the Fisher-Rao geodesic between $\mu$ and $\nu$ given in \eqref{eq:geod-curve-FR}, \ie
        $\omega(s) = \left((1-s) \sqrt{\mu} + s\sqrt{\nu}\right)^2
        ,
        \
        s\in[0,1]$.
    Then,
    $\overline{\FR}^2_{\omega(s)}(\mu,\nu)$ coincides with,
    \begin{enumerate}[itemsep=0pt, parsep=0pt, topsep=0pt, partopsep=0pt, leftmargin=*]
        \item if $s=0$, 
        the $\chi^2$-divergence
        $\mathrm{D}_{\chi^2}(\mu|\nu)=\|\frac{\dd\mu}{\dd\nu}-1\|^2_{L^2_\nu}$;
        \item if $s=1$,
        the reverse $\chi^2$-divergence $\mathrm{D}_{\chi^2}(\nu|\mu)=\|\frac{\dd\nu}{\dd\mu} - 1 \|^2_{L^2_\mu}$;
        \item if $s=\frac12$, 
        the squared Fisher-Rao (Hellinger) distance itself 
        $\FR^2(\mu, \nu)={4}\|\sqrt{\mu}-\sqrt{\nu}\|^2_{L^2}$.
    \end{enumerate}
\label{thm:fr-linearized-static}
\end{proposition}
\begin{remark}
    [Fisher-Rao geodesic and flatness]
    The third case demonstrates a particularly nice property of the Fisher-Rao geometry.
    The Fisher-Rao geometry is not flat and possesses a geodesic structure~\eqref{eq:intro-FR-hamilton}.
    Yet its geodesic distance can be computed by a flat distance $\overline{\FR}_{\omega(\frac12)}(\mu,\nu)$ characterized in Proposition~\ref{thm:fr-linearized-static}.
\end{remark}

\section{Analysis}
\label{sec:analysis}

\subsection{Explicit connection between nonparametric regression and Rayleigh Principle}
\label{sec:nonparametric-regression}

In this section, we uncover a connection between
the entropy dissipation in gradient flows and the nonparametric regression
\begin{align}
    \argmin_{f\in \calF}
    \left\{\int \left(f(x) - y(x)\right)^2 \dd \mu(x)  
    +
    \lambda \|f\|^2_{\calF}\right\}
    ,
    \label{eq:nonparametric-regression-LSQ}
\end{align}
where $y$ is some target functions such as the generalized force $\dFdmu$ for Fisher-Rao and velocity $\nabla \dFdmu$ for Wasserstein.
This can be further unified under the general MLE formulation with negative log-likelihood
\begin{align}
        \argmin_{f\in \calF}\mathrm{NLL}(f; y, \mu)  
            ,
            \label{eq:gfe-krr-FRk-MLE}
\end{align}
    where 
    $\mathrm{NLL}$ is the negative-log-likelihood, e.g., $\frac{1}{2\sigma^2}\|f -  y\|^2_{L^2_{\mu}} +\lambda \|f\|^2_{L^2_{\mu}} $ with variance $\sigma^2$.
We now give an explicit characterization that connects gradient flow dissipation geometry with nonparametric regression.
Since Fisher-Rao and Wasserstein type flows are based on $L^2$ type geometries, we only focus on the least-squares type losses~\eqref{eq:nonparametric-regression-LSQ}.
Our main result in this subsection is a connection between the nonparametric regression~\eqref{eq:nonparametric-regression-LSQ} and the Helmholtz-Rayleigh Principle~\citep{rayleigh_general_1873}.
\begin{proposition}
    [Nonparametric regression as Helmholtz-Rayleigh Principle]
    In approximate Fisher-Rao, Wasserstein, and Wasserstein-Fisher-Rao flows, the nonparametric regression~\eqref{eq:nonparametric-regression-LSQ},
    where the target $y$ is the velocity field of the corresponding (pseudo-)Riemannian manifold,
    is equivalent to the optimization problem
    \begin{align}
        \argmin_{f\in \calF}
        \biggl\{
            {\frac{\dd}{\dd t}F(\mu^f_t)}
            +
            \frac12\|f\|_{L^2_{\mu}}^2
            +\frac{\lambda}2 \|f\|^2_{\calF}
            \biggr\}
            .
            \label{eq:regression-as-dissipation}
    \end{align}
    Furhermore, solving the nonparametric regression~\eqref{eq:nonparametric-regression-LSQ} is equivalent to the minimization problem in the form of the Rayleigh Principle~\citep{rayleigh_general_1873}
    \begin{align*}
        \argmin_{f\in \calF}
        \underbrace{\iprod{\DF}{\frac{\dd}{\dd t}{\mu^f_t}}}_{\textrm{energy}}
        +
        \underbrace{\frac12\|f\|_{L^2_{\mu}}^2
            +\frac{\lambda}2 \|f\|^2_{\calF}}_{\textrm{dissipation potential}}
            .
    \end{align*}
    \label{prop:regression-as-dissipation}
\end{proposition}
Therefore,
the approximation of the gradient flow via nonparametric regression~\eqref{eq:nonparametric-regression-LSQ}
is equivalent to searching for the approximate growth or velocity field $f\in\calF$ that maximally dissipates the energy $F$, while regularized by its function class $\calF$, $L^2$ norm, and optionally a regularization term $\frac{\lambda}2 \|f\|^2_{\calF}$.
\begin{remark}
    [Helmholtz-Rayleigh Principle]
    While our discussion around Proposition~\ref{prop:regression-as-dissipation} is formal, it can be made mathematically rigorous in the form of the (Helmholtz-)Rayleigh Principle~\citep{rayleigh_general_1873}, also known as the maximum dissipation principle; see \citep[Proposition~5.2.1]{Miel15VAMD} for the rigorous statement.
    Our notation for the dissipation geometry $\calR, \calR^*$
    is also due to the Rayleigh dissipation function.
\end{remark}

As discussed in the previous subsection,
the inclusive KL divergence
generates the MMD.
Using Proposition~\ref{prop:regression-as-dissipation}, we show further connection to tools in machine learning in the following examples.
First, Proposition~\ref{prop:regression-as-dissipation} specialized to the Wasserstein setting gives a connection to generative models.
\begin{example}
    [Wasserstein case as implicit score-matching]

    A standard technique for solving a score-matching problem is the
    implicit score-matching (ISM)
    for solving regression
    \begin{align}
         \argmin_{f\in \calF}
            \|f - \nabla \log\frac{\mu}{\pi}\|_{L^2_{\mu}}^2 
            .
            \label{eq:score-matching-implicity}
        \end{align}
    Apply Proposition~\ref{prop:regression-as-dissipation} (with $\lambda=0$) in the Wasserstein setting with the KL entropy energy,
    we recover the ISM result~\citep{hyvarinenEstimationNonNormalizedStatistical,vincentConnectionScoreMatching2011}
    \begin{align}
        \argmin_{f\in \calF}
        \left\{- \iprod{f}{\nabla \log \frac{\mu}{\pi}}_{L^2_{\mu}}^2
        +\|f\|_{L^2_{\mu}}^2 
        \overset{\text{(IBP)}}{=}
        \int
        \left(f^2 +  {\DIV{f} }
        + {{f\cdot \nabla \log \pi} }\right){ {\mu}}\right\}
        .
        \label{eq:score-matching-implicity-2}
    \end{align}
    In practice, the estimator $\hat f$ in \eqref{eq:score-matching-implicity-2} can be fitted using, \eg deep neural networks,
    and used as the velocity field to update the particle locations,
    \ie performing Langevin update
$        X_{t+1} \gets X_t + \tau\cdot  \hat f(X_t)$.
    This technique has also been applied with neural network approximation for sampling by~\citet{dong2022particle}.
\end{example}

Going beyond Wasserstein, specializing Proposition~\ref{prop:regression-as-dissipation} to the Fisher-Rao setting results in a commonly used tool in machine learning applications.
\begin{example}
    [Inclusive KL dissipation in approximate $\FR^2$ and two-sample test]
    Let the energy functional
    be the inclusive KL divergence
    $F(\mu)=\mathrm{D}_{\mathrm{KL}}(\pi |\mu )$.
    Energy-dissipation formulation of nonparametric regression~\eqref{eq:regression-as-dissipation}
    is equivalent to the
    weak-norm formulation
    \begin{align*}
        \sup_{f\in \calF}
        \biggl\{
            \int f
            \dd\left( {{ \mu} - \pi}\right)
            -
            \frac12\|f\|_{L^2_{\mu}}^2
            -
            \frac{\lambda}2 \|f\|^2_{\calF}
            \biggr\}
    \end{align*}
    We have already seen this type of weak-norms in the de-kernelized and flattened geometries.
Specifically, the choice $\calF$ as the RKHS gives a regularized version of the MMD.
Weak-norm formulations,
often termed the integral probability metric (IPM),
are also
commonly used in machine learning
applications such as generative models~\citep{nowozin_f-gan_2016},
two-sample testing~\citep{gretton2012kernel},
and robust learning under distributional shifts~\citep{zhu_kernel_2021}.
\end{example}

Another implication of Proposition~\ref{prop:regression-as-dissipation} is that the Stein gradient flow can not be cast in the form~\eqref{eq:regression-as-dissipation} since kernel smoothing is not an \emph{M-estimator}, \ie it cannot be written as the solution of an optimization problem like \eqref{eq:nonparametric-regression-LSQ}.
However, it is possible to cast it as a local regression; see
texts in nonparametric statistics
\citep{spokoiny2016nonparametric,tsybakov_introduction_2009}.
\begin{example}
    [Approximation in the Stein setting as local regression and MLE]
    The
    particle approximation scheme for the Stein PDE~\eqref{eq:stein-PDE}
    was first proposed as the SVGD algorithm by \citet{liu_stein_2019}.
    We now cast the approximation in the Stein setting as local regression and thus local MLE
    \begin{align}
        \dot {\mu_t} = \DIV ({\mu_t}\cdot v_t), \quad
            v_t(x) = \argmin_{\theta \in \R^d}
            \biggl\{
                    \int \mu(x') k(x'-x) 
                    \biggl|\theta  -  \nabla \dFdmut(x')\biggr|^2
                    \dd x'
                \biggr\}
                ,
        \end{align}
        where the regression is now locally weighted using the shift-invariant kernel $k(x'-x)$.
        This problem admits a closed-form solution
        as a kernelized velocity
        \begin{align}
            v_t(x) = 
            \int \mu(x') \frac{k(x'-x)}{\int \mu(x') k(x'-x) \dd x'} \nabla \dFdmut(x') \dd x'
            .
        \end{align}
    Given data samples $\{x_i\}_{i=1}^N$,
    we obtain the kernel smoothing method known as the Nadaraya-Watson estimator
    $\displaystyle {\hat v(x) = \sum_{i=1}^N \frac{ k(x_i- x) }{\sum_{i=1}^N k(x_i- x)}\cdot  \nabla \dFdmut(x_i)}$, which is the SVGD velocity with a normalized kernel.
\end{example}

\begin{example}
    [Nadaraya-Watson estimator in the kernelized Fisher-Rao setting]
    Similar to Stein, in the kernelized Fisher-Rao setting, we have
    \begin{align}
        \dot {\mu_t} = -{\mu_t}\cdot r_t, \quad
            r_t(x) = \argmin_{\theta \in \R^d}
            \biggl\{
                    \int \mu(x') k(x'-x) 
                    \biggl|\theta  -   \dFdmut(x')\biggr|^2
                    \dd x'
                \biggr\}
                ,
        \end{align}
        with the closed-form solution as a kernelized growth field
        \begin{align}
            r_t(x) = 
            \int \mu(x') \frac{k(x'-x)}{\int \mu(x') k(x'-x) \dd x'}  \dFdmut(x') \dd x'
            ,
            \label{eq:kernelized-growth-field}
        \end{align}
        and a Nadaraya-Watson estimator
    $\displaystyle {\hat r(x) = \sum_{i=1}^N \frac{ k(x_i- x) }{\sum_{i=1}^N k(x_i- x)}\cdot  \dFdmut(x_i)}$.

    In particular, in the setting where the energy is the 
    inclusive KL-divergence,
        which is the
    $0$-th order power entropy dissipation $F(\mu)=\mathrm{D}_0(\mu|\pi )$ in \eqref{eq:kernelized-growth-field}, we obtain growth field
    $\int \frac{k(x'-x)}{\int \mu(x') k(x'-x) \dd x'}  \left( \mu(x') - \pi(x')\right) \dd x'$.
    Given two samples $\{y_i\}_{i=1}^N\sim \mu $ and $\{z_i\}_{i=1}^M\sim \pi$,
    a sample-based estimator of the growth field is the difference between two kernel density estimators
    \begin{align*}
        \hat r(x) =
            \sum_{i=1}^N \frac{ k(y_i- x) }{ \sum_{i=1}^N k(y_i- x)}
            -
            \sum_{i=1}^M\frac{ k(z_i- x) }{ \sum_{i=1}^M k(z_i- x)}
        .
    \end{align*}
\end{example}

\subsection{Evolutionary $\Gamma$-convergence at the approximation limit}
\label{sec:g-cvgs}
As
the regularization parameter $\lambda\to 0$,
techniques from
approximation theory and statistics can be used to show
consistency for \emph{fixed time $t$ point-wise}.
That is, the regression problems such as \eqref{eq:gfe-krr-FRk}~\eqref{eq:KRR-stein} yield the velocity $v_t$ or growth $r_t$ that converges to the target counterpart,
\eg $v_t\overset{\lambda\to 0^+}{\longrightarrow} \nabla \dFdmut$ for fixed $t$.
However, this point-wise convergence \emph{does not directly imply the convergence or the existence} of the approximating gradient systems.
For example, it has not been shown that
the approximate system generating $\dot {\mu_t} = -{\mu_t}\cdot r_t$ in \eqref{eq:gfe-krr-FRk} converges to the pure Fisher-Rao gradient system.

Different from the point-wise convergence of regression problems, we now rigorously justify this approximation limit of the gradient systems using
evolutionary
$\Gamma$-convergence in this section.
We focus on the kernel approximation to the Fisher-Rao gradient system in Section~\ref{sec:kern-FR};
see \eqref{eq:gfe-krr-FRk}, \eqref{eq:fr-kernel-reg-dual}.
Note that the rigor level of this subsection's analysis is elevated above the rest of the paper, \ie not merely formal arguments.
A rigorous result for the approximate Wasserstein(-Fisher-Rao) setting is beyond our current scope due to the technicality involved.

For general curves $u:[0,T]\to X$, where $X$ denotes the state space, e.g. the non-negative measures $\Mplus$.
We define the dissipation functional $\mfD_{\lambda}$ as
\begin{align*}
    \mfD_{\lambda}(u):=\int_0^T\left(
        \mathcal{R}_{\lambda}(u(t), \dot{u}(t))+\mathcal{R}_{\lambda}^*\left(u(t),-\mathrm{D} {F}(u(t))\right)
        \right) \mathrm{d} t
    ,
    \\
    \mfD_0(u)=\int_0^T\left(\mathcal{R}(u(t), \dot{u}(t))+\mathcal{R}^*\left(u(t),-\mathrm{D} {F}(u(t))\right)\right) \mathrm{d} t
    .
\end{align*}
The energy-dissipation principle (EDP) states that, under suitable technical assumptions (cf.\ \cite[Theorem~3.9]{mielke2023introduction}), $u:[0,T] \to X$ 
is a solution to the gradient-flow equation \eqref{eq:GFE} if and only if it
satisfies the following energy-dissipation inequality:
\begin{equation}
    \label{eq:EDP.general}
    F(u(t)) + \mfD_{\lambda}(u) \leq F(u(0)). 
\end{equation}
Thus, $\mfD_\lambda$ intrinsically encodes the gradient-flow dynamics.

\begin{definition}
    [EDP-convergence]
    A sequence of gradient systems $(X,F,\calR_\lambda)$ is said to 
    \emph{converge in the sense of the energy-dissipation principle} (EDP-converge) to $(X,F,\calR)$, shortly written as
    $(X,F,\calR_\lambda) \EDPto (X,F,\calR)$,
    if $\mfD_{\lambda}$ $\Gamma$-converges to $\mfD_{0}$ with bounded energies for all $T>0$,
    \ie
    \begin{align}
       (\Gamma\mafo{inf})\quad & u_{\lambda} \to u \text{ and } 
        \sup_{\lambda>0, 0\leq t\leq T} {F}\left(u_{\lambda}(t)\right)<\infty
        \implies
        \liminf_{\lambda\to 0^+} \mfD_{\lambda} \left(u_{\lambda}\right) \geq \mfD_0(u),
        \label{eq:edp-conv-inf}
        \\
         (\Gamma\mafo{sup})\quad & \forall\, \widehat{u} \in L^2\left([0,T]; X\right) \ \exists\ \widehat{u}_{\lambda}\text{ with }
        \sup_{\lambda>0, \: 0\leq t\leq T} {F}\left(\widehat{u}_{\lambda}(t)\right)<\infty, 
        :
        \nonumber
        \\
        &
        \quad\quad\quad\quad\quad\quad
        \quad\quad\quad\quad\quad\quad
          \widehat{u}_{\lambda} \rightarrow \widehat{u} \text { and } \limsup _{\lambda \rightarrow 0} \mfD_{\lambda}\left(\widehat{u}_{\lambda}\right) \leq  \mfD_0(\widehat{u}).
        \label{eq:edp-conv-sup}
    \end{align}

\end{definition}
Recall that the dissipation geometry of the regularized-approximate Fisher-Rao gradient system in Section~\ref{sec:kern-FR}.
\begin{align*}
    \calR_{\LFRk} (\rho,  u) =
    \frac{1}{2}
    \iprod{\frac{\delta u}{\delta \rho}}
    {
        \left(
            \ID {+} \lambda \K_\rho^{-1}
        \right)
        \frac{\delta u}{\delta \rho}
        }_{L^2_{\rho}}
    , \ \ 
    \calR_{\LFRk}^* (\rho,  \xi) =
    \frac{1}{2}
    \iprod{\xi}
    {
        \left(
            \K_\rho {+} \lambda \ID
            \right)
            ^{-1}
            \K_\rho
        \xi
        }_{L^2_{\rho}}
    .
\end{align*}
A useful observation is that $\calR_{\LFRk} (\rho,  u)$ is decreasing with decreasing $\lambda$. In fact, it is even affine. Since the Legrendre transform is anti-monotone, $\calR_{\LFRk} (\rho,  u)$ is increasing for $\lambda$ decreasing to $0$. To avoid technicalities, we do not show full EDP-convergence, but only the $\Gamma$-liminf estimate \eqref{eq:edp-conv-inf}, this means we stay in the $\Gamma$-convergence framework of \cite{Serf11GCGF}. This will be enough to conclude that solutions $\mu_\lambda$ of the regularized gradient-flow equation of $(\Mplus,F,\calR_{\LFRk})$ converge to solutions $\mu$ of the pure Fisher-Rao gradient-flow equation of $(\Mplus,F,\calR_{\FR})$, see Corollary \ref{cor:Cvg.GradFlowSol}.

\begin{theorem}[$\Gamma$-convergence of the kernel-approx. FR gradient systems]
\label{pr:EDPcvg.FR}
\mbox{ } \\
Assume that the functional $F:\Mplus\to \R$ satisfies the following assumptions:
\begin{align}
    &\text{the $\LFRk$-dissipation } \mu \mapsto \calR_{\LFRk} \big( \mu,\tfrac{\delta F}{\delta\mu}(\mu)\big) 
     \text{ is weakly lower semicontinuous}.
    \label{eq:EDP.Ass1}
\end{align}
    Then, the dissipation functional $\mfD_\lambda$ for the regularized approximate Fisher-Rao gradient system
    $(\Mplus, F,  \calR_{\LFRk})$ satisfies the $\Gamma$-liminf estimate \eqref{eq:edp-conv-inf}. 
\end{theorem}
The above result allows us to conclude that
the solution of the approximate Fisher-Rao gradient flows
converges to that of the pure Fisher-Rao.
\begin{corollary}[Convergence of Gradient flow solutions]
\label{cor:Cvg.GradFlowSol}
    Let $\mu_\lambda$ be a sequence of solutions to the regularized Fisher-Rao gradient system $(\Mplus, F,  \calR_{\LFRk})$ in the sense of 
    energy-dissipation balance. Assume that the assumptions of 
    Theorem~\ref{pr:EDPcvg.FR} are satisfied.
    Suppose that for all 
    $t\in [0, T]$ we have $ \mu_\lambda(t) \rightharpoonup  \mu(t)$
    and that $F(\mu_\lambda(0)) \to F(\mu(0)) < \infty$. 
    
    Then, $ \mu:[0,T] \to (\Mplus,\FR)$ is absolutely continuous and a solution to the Fisher-Rao gradient system $(\Mplus, F,\FR)$.
\end{corollary}
\begin{proof} 
    [Proof of Corollary~\ref{cor:Cvg.GradFlowSol}]
    By Theorem~\ref{pr:EDPcvg.FR} we know that $ \mu$ satisfies 
\[
\int_0^T | \dot \mu|_\FR(t)^2 \dd t < \infty \quad \text{and} 
\quad F(\mu(T)) + \mfD_0( \mu) \leq F(\mu(0)). 
\]
The last relation follows by the EDP for $\mu_\lambda$, namely $F(\mu_\lambda(T)) + \mfD_\lambda( \mu_\lambda) \leq F(\mu_\lambda(0))$ (see \eqref{eq:EDP.general}) and the limit passage $\lambda \downarrow 0$. Now exploiting the EDP for $\lambda =0$ we see that $\mu$ is a solution for $(\Mplus,F,\calR_\FR)$. 
\end{proof}
Thus far, we have answered the question we posed earlier:
the approximate flow,
using a regression formulation such as \eqref{eq:gfe-krr-FRk-npreg},
is indeed a gradient flow that converges to the target gradient-flow system such as the Fisher-Rao gradient flow,
in the sense of evolutionary $\Gamma$-convergence.
This provides the mathematical basis for ``learning'' the flow for machine learning applications.

\section{Further proofs}
\label{sec:proof}
\begin{proof}
    [Proof of Lemma~\ref{thm:mmd-variational-form}]    
By definition,
    $\mmd ^p (\mu, \nu) =  
    \|\K\mu - \K \nu\|^p_\rkhs.$
We introduce the auxiliary variable $f = \int k(x, \cdot )\dd \mu(x)$, then apply the Lagrange duality to the constrained optimization problem
\begin{align*}
    \inf_{f\in\rkhs}
    \|f - \K \nu\|^p_\rkhs 
    \ \ST\ f = \int k(x, \cdot )\dd \mu(x).
\end{align*}
Finally, we associate the equality constraint $f = \int k(x, \cdot )\dd \mu(x)$ with the dual variable $h\in\mathcal H$.
By the first order optimality condition,
\begin{align}
    2 \cdot (f - \K \nu) = -h,
    \
        f = \K \nu  -\frac1{2}h.
\end{align}
Hence,
\begin{align}
         \mmd ^2 (\mu, \nu) 
        &=  
    \sup_{h\in\rkhs}
    \int{h}\dd{(\nu-\mu)}
    +
    \|\frac{h}{2}\|^2_\rkhs
    -\frac1{2}
    \|{h}\|^2_\rkhs 
    \\
    &=
    \sup_{h\in\rkhs}
    \int{h}\dd{(\nu-\mu)}
    -\frac1{4}
    \|{h}\|^2_\rkhs 
    .
\end{align}
The optimizing $h^*$ can be further obtained by directly solving the quadratic program.
\end{proof}
For aesthetic reasons, we need the following lemma whose derivation is an exercise in convex analysis.
\begin{lemma}
    For a scaling parameter $\tau>0$, we find
\begin{align}
\frac1{2\tau}\cdot \mmd^2( \mu, \nu)
=
\sup_{h\in\rkhs}
\intprod{h}{ (\mu- \nu)}
-\frac{\tau}{2}\|h\|^2_\rkhs
            .
            \label{eq:ipm-dual-constr-two-tau}
\end{align}
\label{lmm:ipm-dual-scaling}
\end{lemma}

\begin{proof}
    [Proof of Theorem~\ref{thm:mmd-fr-dynamic}]
    We derive the force-kernelized MMD gradient flow, i.e., we replace the differential
    $\xi_t$ with its kernelization $\K_{\mu_t}^\frac12 \xi_t$ by Definition~\ref{def:f-v-kernelize}.
    The boundary conditions of \eqref{eq:bb-formula-mmd} are trivially equivalent to $\mu(0)=\nu, \mu(1)=\mu$.
    Following the dynamic definition of the Stein distance,
    the cost of the trajectory optimization problem in the MMD formulation $\frac12\|  \xi_t \|^2_{\rkhs}=\frac12\iprod{\xi_t}{\K^{-1} \xi_t}_{L^2_\mu}$,
    as well as the dual dissipation potential, should be replaced with an additional (weighted) $\K_{\mu_t}$ operation by 
    \begin{align*}
    \frac12\iprod{\K^{-\frac12} \K_{\mu_t}^\frac12 \xi_t}{ \K^{-\frac12} \K_{\mu_t}^\frac12 \xi_t}_{L^2}=
    \frac12\|   \xi_t \|^2_{L^2(\mu_t)} 
    .
    \end{align*}
    Using this gradient structure,
    the gradient-flow equation is obtained as
    $$\K \dot \mu = -    \K_{\mu_t} \xi_t.$$
    Since the convolutional operator $\K$ is positive definite,
    this gradient-flow equation is equivalent to the reaction equation in \eqref{eq:bb-formula-fr}, \ie
    $\dot \mu = -   \mu_t\cdot   \xi_t$.
    Therefore, we have recovered the Fisher-Rao geodesics.
\end{proof}

\begin{proof}
    [Proof of Proposition~\ref{thm:fr-linearized-static}]
    If $s=0$, then $\omega(0)=\nu$.
    By Fenchel-duality,
    the linearized FR can be written as
    \begin{align*}
        \sup_{\zeta}
            \int \zeta\dd (\mu - \nu)
            -\frac14
            \|\zeta\|^2_{L^2_\nu}
            =
            \sup_{\zeta}
            \int 
            \left(
                \zeta
                \cdot 
            (\frac{\dd \mu}{\dd\nu}-1)
            -\frac14\zeta^2
            \right)
            \dd \nu
            \\
            =
            \int
            \left(
                \frac{\dd \mu}{\dd \nu} - 1
                \right)^2
            \dd \nu
            =\mathrm{D}_{\chi^2}(\mu|\nu)
            ,
    \end{align*}
    hence the equivalence to the $\chi^2$-divergence.    
    The case of $s=1, \omega(1)=\mu$ is similar.

    If $s=\frac12$, then $\omega(\frac12)=\frac14(\sqrt{\mu}+\sqrt{\nu})^2$.
    We find
    \begin{multline}
        \sup_{\zeta}
            \int 
            \left(
                \zeta
                \cdot 
            \frac{\sqrt{\mu}-\sqrt{\nu}}{\sqrt{\mu}+\sqrt{\nu}}
            -\frac1{16}\zeta^2
            \right)
            \cdot 
            \left(
                \sqrt{\mu}+\sqrt{\nu}
            \right)^2
            \dd x
            \\
            =
            \int
            4\left(
                \frac{\sqrt{\mu}-\sqrt{\nu}}{\sqrt{\mu}+\sqrt{\nu}}
                \right)^2
                \cdot 
                \left(
                    \sqrt{\mu}+\sqrt{\nu}
                \right)^2
                \dd x
            =4\|\sqrt{\mu}-\sqrt{\nu}\|^2_{L^2}
            =\FR^2(\mu,\nu)
            .
    \end{multline}
\end{proof}

\begin{proof}
    [Proof of Proposition~\ref{thm:sobolev-ipm}]
Let $\nabla \zeta_t\in \nabla C^\infty_0$ be the test function with zero boundary condition.
Taking time derivative along the flow solution,
\begin{multline}
\frac{\dd}{\dd t}\int \zeta_t \mu_t 
\overset{\textrm{(product)}}{=} 
\int\partial_t \zeta_t \mu_t + \int \zeta_t \partial_t\mu_t 
\overset{\textrm{(dynamics)}}{=}
\int\partial_t \zeta_t \mu_t - \int \zeta_t \DIV ( \K^{-1}\nabla \xi)
\\
\overset{\textrm{(IBP)}}{=}
\int\partial_t \zeta_t \mu_t + \int \int \K^{-\frac12}\nabla \zeta_t  \K^{-\frac12}\nabla \xi \dd x\dd t
.
\label{eq:time-der-test-otto-villani}
\end{multline}

Completing the squares for the last term, we find
\begin{multline*}
    \int \int \K^{-\frac12}\nabla \zeta_t  \K^{-\frac12}\nabla \xi  \dd x\dd t
    = \int \dd t
    \biggl[
    \int \K^{-\frac12}\nabla \zeta_t  \K^{-\frac12}\nabla \xi  \dd x
    \\
    -\frac{1}{2}
    \left(
    \|\K^{-\frac12}\nabla \zeta_t\|^2_{L^2} + \|\K^{-\frac12}\nabla \xi_t\|^2_{L^2}
    \right)
    +\frac{1}{2}
    \left(
    \|\K^{-\frac12}\nabla \zeta_t\|^2_{L^2} + \|\K^{-\frac12}\nabla \xi_t\|^2_{L^2}
    \right)
    \biggr]
    \\
    =
    \frac12
    \int \dd t
    \left[
    -\|\nabla \zeta - \nabla \xi \|^2_\rkhs
    +
    \| \nabla\zeta_t\|^2_{\rkhs} + \|\nabla\xi_t\|^2_{\rkhs}
    \right]
    .
\end{multline*}
Integrating \eqref{eq:time-der-test-otto-villani} w.r.t. time $t$ and rearranging the terms
\begin{multline}
    \frac12\int  \|\nabla\xi_t\|^2_{\rkhs}\dd t
    \\
    =
    \int \zeta_1\dd \mu_1 - \zeta_0\dd \mu_0
    -
    \int \int
        \partial_t \zeta_t \cdot \mu_t
        \dd x \dd t
        -\frac12
        \int \|\nabla \zeta_t\|^2_{\rkhs}
        \dd t
        +
        \frac12 
        \int \|\nabla \zeta - \nabla \xi \|^2_\rkhs\dd t
        .
        \label{eq:pf-otto-villani-2}
\end{multline}
We now consider the duality in the optimization problem of the BB-formula, \ie
\begin{align*}
    \inf_{\xi,\mu}
    \frac12\int  \|\nabla\xi_t\|^2_{\rkhs}\dd t
    =
    \sup_\zeta
    \inf_{\xi,\mu} \left(\textrm{RHS of }\eqref{eq:pf-otto-villani-2}\right)
\end{align*}
The crucial feature of \eqref{eq:pf-otto-villani-2}
is that
the term $\frac12\int \|\nabla \zeta_t\|^2_{\rkhs}
\dd t$ is independent of the measure $\mu_t$.
Therefore, in order for the infimum w.r.t. $\mu$ to be finite, we require the condition $\partial_t \zeta_t\le 0$ to hold.
At optimality, we recover the adjoint equation in the Hamiltonian dynamics
$$\partial_t \zeta= 0
,
$$
\ie $\zeta$ is a time-independent (static) function.
Hence, \emph{the optimization problem is greatly simplified
to a static setting}, which is in contrast to the Wasserstein and Stein settings.
We find 
\begin{align*}
    \frac12\cdot 
    \inf_{\xi,\mu}
    \int  \|\nabla\xi_t\|^2_{\rkhs}\dd t
    =
    \sup_\zeta
    \left\{
        \int \zeta\dd (\mu_1 - \mu_0)
        -\frac12
         \|\nabla \zeta\|^2_{\rkhs}
    \right\}
\end{align*}
Noting the scaling in the dual formulation from Lemma~\ref{lmm:ipm-dual-scaling},
we find that the transport cost coincides with the static Kantorovich dual formulation.
\end{proof}

\begin{proof}
    [Proof of Proposition~\ref{prop:kernel-discrepancies} and Proposition~\ref{prop:stein-dissipation}]
    For the $\varphi$-divergence energy dissipation in kernelized Fisher-Rao gradient flow,
    \begin{multline*}
        \calI^{\FRk}_\varphi(\mu|\pi)
        =
        -\frac{\dd}{\dd t} \mathrm{D}_\varphi(\mu\|\pi)
        =
        -
        \iprod{\varphi^\prime\left(\frac{\dd \mu}{\dd \pi }\right)}{ \K_\mu  \varphi^\prime\left(\frac{\dd \mu}{\dd \pi }\right)}_{L^2_\mu}
        \\
        \overset{\textrm{(IBP)}}{=}
        \int
        \int
        \frac{\dd \mu}{\dd \pi }(x)
        \cdot 
        \varphi^\prime\left(\frac{\dd \mu}{\dd \pi }(x)\right)
        k(x,y)
        \frac{\dd \mu}{\dd \pi }(y)
        \cdot
        \varphi^\prime\left(\frac{\dd \mu}{\dd \pi }(y)\right)
        \dd \pi (x)
        \dd \pi (y)
        .
    \end{multline*}
    For the $\varphi$-divergence energy dissipation in Stein gradient flow,
    \begin{multline*}
        \calI^{\operatorname{Stein}}_\varphi(\mu(t)|\pi)
        =
        -\frac{\dd}{\dd t} \mathrm{D}_\varphi(\mu(t)\|\pi)
        =
        -\iprod{\varphi^\prime\left(\frac{\dd \mu}{\dd \pi }\right)}
        { \DIV \left(\mu \K_\mu \nabla  \varphi^\prime\left(\frac{\dd \mu}{\dd \pi }\right)\right)}_{L^2}
        \\
        \overset{\textrm{(IBP)}}{=}
        \int\int
        \frac{\dd \mu}{\dd \pi }(x)
        \nabla \varphi^\prime\left(\frac{\dd \mu}{\dd \pi }(x)\right)
        k(x,y)
        \frac{\dd \mu}{\dd \pi }(y)
        \nabla \varphi^\prime\left(\frac{\dd \mu}{\dd \pi }(y)\right)
        \dd \pi (x)
        \dd \pi (y)
        .
    \end{multline*}
    Other calculation is straightforward.
\end{proof}

\begin{proof}
    [Proof of Proposition~\ref{prop:regression-as-dissipation}]
    The proof is formal.
    Consider the Fisher-Rao type approximate flows
    $\displaystyle {{\dot\mu}_t^f = - {\mu_t^f}\cdot f_t} $ where the approximate growth field $f$ is obtained by solving the nonparametric regression~\eqref{eq:nonparametric-regression-LSQ}.
    Here, the target of the regression is $\DF$, the derivative of energy $F$ in the respective geoemtry, e.g., in the sense of Fr\'echet.
    Then,
    \begin{multline*}
            \|f - \DF \|_{L^2_{\mu^f}}^2  +\lambda \|f\|^2_{\calF}
            =
            \|f\|_{L^2_{\mu^f}}^2 - 2 \iprod{f}{\DF}_{L^2_{\mu^f}}
            + \|\DF\|_{L^2_{\mu^f}}^2 +\lambda \|f\|^2_{\calF}
            \\
            =
            2\cdot \frac{\dd}{\dd t}F(\mu^f_t)+
            \|f\|_{L^2_{\mu^f}}^2  
            +\lambda \|f\|^2_{\calF}
            + \|\DF\|_{L^2_{\mu^f}}^2 
            .
        \end{multline*}
    In the Wasserstein type approximate flows
        $\displaystyle \dot {\mu_t} = \DIV (\mu_t\cdot {f_t})$,
        a similar derivation with IBP
        yields
        \begin{align*}
                \|f - \nabla \DF \|_{L^2_{\mu^f}}^2  +\lambda \|f\|^2_{\calF}
                =
                \|f\|_{L^2_{\mu^f}}^2 - 2 \iprod{f}{\nabla\DF}_{L^2_{\mu^f}} + \|\DF\|_{L^2_{\mu^f}}^2 +\lambda \|f\|^2_{\calF}
                \\
                \overset{\mathrm{(IBP)}}{=}
                2\cdot \frac{\dd}{\dd t}F(\mu^f_t)+
                \|f\|_{L^2_{\mu^f}}^2  
                +\lambda \|f\|^2_{\calF}
                + \|\nabla \DF\|_{L^2_{\mu^f}}^2 
                .
            \end{align*}
    Combining those two cases, we obtain the similar formulation for the Wasserstein-Fisher-Rao type flows,
        \begin{multline*}
            \|g - \DF \|_{L^2_{\mu^{g,h}}}^2 + \|h - \nabla \DF \|_{L^2_{\mu^{g,h}}}^2  +\lambda (\|g\|^2_{\calF}+\|h\|^2_{\calF})
                \\=
                2\cdot \frac{\dd}{\dd t}F(\mu^{g,h}_t)
                +\|g\|_{L^2_{\mu^{g,h}}}^2  
                +\|h\|_{L^2_{\mu^{g,h}}}^2  
                +\lambda  (\|g\|^2_{\calF}+\|h\|^2_{\calF})
            + \|\DF\|_{L^2_{\mu^f}}^2 
                + \|\nabla \DF\|_{L^2_{\mu^{g,h}}}^2 
                .
            \end{multline*}
    By setting the  product $f=g\otimes h$, we
    observe that the formal calculation in different geometries above results in the same optimization objectives on the right-hand sides independent of the dissipation geometries.
\end{proof}

\begin{proof}
    [Proof of Proposition~\ref{prop:local-Lojas-estimate-fine}]
    For constant $0\leq s \leq 1$ and a set of orthonormal
bases $\left\{e_j\right\}$ of ${L^2_{\mu}}$,
\begin{multline}
    \frac{\dd}{\dd t}
    \left(
        F(\mu(t)) - \inf_\mu F(\mu)
    \right) 
= 
-\langle{\dFdmu}, \dot u \rangle_{L^2_{\mu_t}}
    =
-
\iprod{\dFdmu}
    {
        \left(
            \K_\mu, + \lambda \ID
            \right)
            ^{-1}
            \K_\mu,
        \dFdmu
        }_{L^2_{\mu}}
\\
    =
    -\|{\dFdmu}\|^2_{L^2_{\mu}}
    +
    \sum_j
    \frac{\lambda}{\sigma_j + \lambda}
    \cdot
    \left| 
\iprod{\dFdmu}{e_j}
\right|^2
    =
    -\|{\dFdmut}\|^2_{L^2_{\mu}}
    \\
    +
    {\lambda}^s
    \sum_j
    \left(\frac{\lambda}{\sigma_j + \lambda}\right)^{1-s}
    \cdot
    \frac{
        \left| 
            \iprod{\dFdmut}{e_j}
        \right|^2
        }
        {\left(\sigma_j + \lambda\right)^{s}}
        \leq 
        -\|{\dFdmut}\|^2_{L^2_{\mu}}
    +
    {\lambda}^s
        \|
        \left(\K_\mu + \lambda \ID \right)^{-\frac{s}{2}}\ \dFdmut
        \|^2_{L^2_{\mu}}
    .
    \label{eq:approx-Lojas-fine}
    \end{multline}
\end{proof}

\begin{proof}
    [Proof of Theorem~\ref{pr:EDPcvg.FR}]
As in \cite{Serf11GCGF} we decompose $\mfD_\lambda=\mfD^\mafo{rate}_\lambda+\mfD^\mafo{slope}_\lambda$ into a rate part and a slope part: 
\[
\mfD^\mafo{rate}_\lambda(\mu)=\int_0^T\calR_{\LFRk}(\mu,\dot \mu)\dd t \quad\text{and} \quad 
\mfD^\mafo{slope}_\lambda(\mu)=\int_0^T\calR^*_{\LFRk} \big(\mu,
      \tfrac{\delta F}{\delta \mu}(\mu)\big)\dd t .
\]
\underline{(A) Extraction of a converging subsequence.} 
By standard arguments, it is sufficient to consider a family $(\mu_\lambda)_{\lambda>0}$ of curves with 
$\mfD_\lambda(\mu_\lambda) \leq C < \infty$. Using $\calR_{\LFRk}\geq \calR_{\FR}$ we obtain $\int_0^T |\mu'_\lambda|^2_\FR(t) \dd t \leq C$. This implies that there exists a subsequence (not relabeled) and a limit curve $\mu_0$ with
\begin{equation}
    \label{eq:mu.Extract}
\int_0^T |\mu'_0|^2_\FR(t) \dd t \leq C \quad \text{and} \quad 
\forall\, t\in [0,T]:\ \mu_\lambda(t) \rightharpoonup \mu_0(t), 
\end{equation}
where weak convergence is meant in the sense of testing with continuous functions.
\medskip

\noindent\underline{(B) $\Gamma$-liminf estimate for $ \mfD^\mafo{rate}_\lambda$.} For this we exploit $ \calR_{\LFRk} \geq \calR_\FR=\frac12|\mu'|^2_\FR$:
\[
\liminf_{\lambda\downarrow 0} \mfD^\mafo{rate}_\lambda(\mu_\lambda)\geq 
\liminf_{\lambda\downarrow 0} \mfD^\mafo{rate}_0(\mu_\lambda)= 
\liminf_{\lambda\downarrow 0} \int_0^T \!\frac12|\mu'_\lambda|(t)^2 \dd t \geq 
\int_0^T \!\frac12|\mu'_0|(t)^2 \dd t = \mfD^\mafo{rate}_0(\mu_0).  
\]
\noindent\underline{(C) $\Gamma$-liminf estimate for $ \mfD^\mafo{slope}_\lambda$.} 
To treat the slope term we use the monotonicity $\calR^*_{\LFRk} \geq \calR^*_{\DFRk} $ for $0<\lambda \leq \delta$ and the weak lower semi-continuity assumed in \eqref{eq:EDP.Ass1}. For fixed $\delta>0$ we have 
\begin{align*}
 \liminf_{\lambda \downarrow 0} 
     \calR^*_{\LFRk}(\mu_\lambda,\tfrac{\delta F}{\delta \mu}(\mu_\lambda)) 
 & \geq \liminf_{\lambda \downarrow 0} 
    \calR^*_{\DFRk}(\mu_\lambda,\tfrac{\delta F}{\delta \mu}(\mu_\lambda)) 
 \geq \calR^*_{\DFRk}(\mu_0,\tfrac{\delta F}{\delta \mu}(\mu_0)),
\end{align*}
where the last estimate follows because of \eqref{eq:EDP.Ass1}. Integration over $t\in [0,T]$, Fatou's lemma yields 
\[
\liminf_{\lambda\downarrow 0} \mfD^\mafo{slope}_\lambda (\mu_\lambda) \geq \int_0^T 
\liminf_{\lambda \downarrow 0} 
     \calR^*_{\LFRk}(\mu_\lambda,\tfrac{\delta F}{\delta \mu}(\mu_\lambda)) \dd t 
  \geq \mfD^\mafo{slope}_\delta( \mu_0) \overset{\delta \downarrow 0}\longrightarrow  \mfD^\mafo{slope}_0(\mu_0), 
\]  
where the last convergence follows by the monotone convergence principle. 

Hence, the desired $\Gamma$-liminf estimate for $ \mfD_\lambda=\mfD^\mafo{rate}_\lambda+\mfD^\mafo{slope}_\lambda$ is established. 
\end{proof}

\section{Discussion}
\label{sec:conclude}
Historically,
geometries over probability measures such as optimal transport and information geometry had a tremendous impact on computational algorithms in optimization and statistical inference.
It is our hope that
the new geometric structure studied in this paper can motivate downstream applications and further investigations.

For example,
in this paper, we used kernel approximation due to its simplicity and well-established approximation theory.
In modern machine learning,
practitioners often use nonlinear models such as deep neural networks
and scalable approximations of kernel machines.
With our nonparametric regression formulation,
it is straightforward to use those learning models.

In terms of generative models,
we observe a direct correspondence of
our nonparametric regression
formulation in Proposition~\ref{prop:regression-as-dissipation}
to generative models such as the flow-matching~\citep{lipman2022flow} and score-matching~\citep{song2020score} algorithms.
The advantage of our characterization using the Rayleigh Principle in Proposition~\ref{prop:regression-as-dissipation}
is that it does not designate a specific geometry,
\ie matching of quantities such as score, velocity, force, and growth
can be expressed in the same formalism of the Rayleigh Principle.

\appendix
\section{Other technical details}

\subsection*{Kernel discrepancies and entropy dissipation in kernelized geometries}
We now show that
the entropy dissipation in the kernelized gradient flows generates a class of discrepancies
meaningful for machine learning applications, in the form of interaction energy.
\begin{proposition}
    [Discrepancies via entropy dissipation in kernelized FR]
    \label{prop:kernel-discrepancies}
    The\\
    dissipation of energy $F$ in the kernelized Fisher-Rao gradient flow ($\FRk$)
    is an interaction energy and a kernel discrepancy between measures:
    \begin{align}
        \calI^{\FRk}_F(\mu)
        =
        \int \int  \dFdmu(x) \ k(x,y)\ \dFdmu(y)
        \dd \mu(x)
        \dd \mu(y)
        .
        \label{eq:energy-dissipation-kernelized-fr}
    \end{align}
    If the energy is
    the $\varphi$-divergence energy, \ie $F(\mu)=\mathrm{D}_\varphi(\mu |\pi)$,
    where
    \begin{equation}
        D_\varphi(\mu | \pi) = \int  \varphi\left(\frac{\dd\mu}{\dd\pi}(x)\right) \dd\pi(x), \text{ if } \mu \ll \pi
        .
        \label{eq:general-divergence}
    \end{equation}
    Then, the dissipation is
    \begin{align}
        \calI^{\FRk}_\varphi(\mu|\pi)
        =
        \int
        \int
        \frac{\dd \mu}{\dd \pi }(x)
        \cdot 
        \varphi^\prime\left(\frac{\dd \mu}{\dd \pi }(x)\right)
        k(x,y)
        \frac{\dd \mu}{\dd \pi }(y)
        \cdot
        \varphi^\prime\left(\frac{\dd \mu}{\dd \pi }(y)\right)
        \dd \pi (x)
        \dd \pi (y)
        .
    \end{align}
Specifically,
we find the commonly used entropy dissipation as kernel discrepancies:
\begin{enumerate}[itemsep=0pt, parsep=0pt, topsep=0pt, partopsep=0pt]
    \item the inclusive KL energy $F(\mu)=\mathrm{D}_{\mathrm{KL}}( \pi| \mu)$,
    we obtain $\mmd^2(\mu, \pi )$ as its dissipation
    \begin{align}
        \calI^{\FRk}_0(\mu | \pi ) &=    
        \int
        \int
        \left(\mu(x)-\pi(x)\right)
        k(x,y)
        \left(\mu(y)-\pi(y)\right)
        \dd x
        \dd y,
        \label{eq:rev-kl-generates-mmd-dissipation}
    \end{align}
    \item the squared Fisher-Rao energy $F(\mu)=\FR^2(\mu, \pi)$, 
    \begin{align}
        \calI^{\FRk}_{\frac12}(\mu | \pi ) &=    
        4
        \int
        \int
        \left(\sqrt{\mu(x)}-\sqrt{\pi(x)}\right)
        k(x,y)
        \left(\sqrt{\mu(y)}-\sqrt{\pi(y)}\right)
        \dd \sqrt{\mu}(x)
        \dd \sqrt{\mu}(y),
    \end{align}
    \item the KL-entropy energy $F(\mu)=\mathrm{D}_{\varphi_1}(\mu|\pi)$,
    \begin{align}
        \calI^{\FRk}_1(\mu | \pi ) &=    
        \int
        \int
        \log\frac{\dd \mu}{\dd \pi }(x)
        k(x,y)
        \log\frac{\dd \mu}{\dd \pi }(y)
        \dd \mu(x)
        \dd \mu(y),
    \end{align}
    \item the $\chi^2$-divergence energy $F(\mu)=\mathrm{D}_{\chi^2}(\mu|\pi)$,
    \begin{align}
        \calI^{\FRk}_2(\mu | \pi ) &=    
        \int
        \int
        \left(\frac{\dd \mu}{\dd \pi }(x)-1\right)^2
        k(x,y)
        \left(\frac{\dd \mu}{\dd \pi }(y)-1\right)^2
        \dd \mu(x)
        \dd \mu(y).
    \end{align}
\end{enumerate}
\end{proposition}
The MMD enjoys a computational advantage as it allows Monte Carlo estimation and does not require the measures $\mu,\pi$ to have common support.
\eqref{eq:rev-kl-generates-mmd-dissipation} reveals the insight that
this is due to the inclusive KL-entropy dissipation structure.
Furthermore, we obtain an interesting insight regarding optimization
\begin{lemma}
    The gradient flow equation generated by the pure Fisher-Rao gradient system with the squared MMD energy $\left(\Mplus, \mmd^2(\cdot, \pi ), \FR\right)$ conincides with that of the kernelized Fisher-Rao gradient system with the inclusive KL-entropy energy $\left(\Mplus, \mathrm{D}_{\varphi_0}(\cdot|\pi), \FRk\right)$, \ie
    \begin{align*}
        \dot \mu = -\mu \cdot \K \left(\mu - \pi \right)
        .
    \end{align*}
    \label{thm:fr-mmd-0th-kfr}
\end{lemma}
This means that solving the optimization problem
$\displaystyle\min_\mu \mmd^2(\mu, \pi)$
in the pure Fisher-Rao geometry is equivalent to
$\displaystyle\min_\mu \mathrm{D}_{\varphi_0}(\mu|\pi)$ in the kernelized Fisher-Rao geometry.

It is already established that the dissipation of the KL-divergence in the Stein geometry results in the Stein-Fisher information~\citep{duncan2019geometry}, also referred to as the
squared kernel Stein discrepancy (KSD)~\citep{liu_kernelized_2016,gorham2017measuring,chwialkowskiKernelTestGoodness2016}.
First, we generalize this result.
\begin{proposition}
    [Kernel discrepancies via entropy dissipation in Stein]
    \label{prop:stein-dissipation}
    Energy dissipation in the Stein gradient flow follows
    \begin{align}
        \calI^{\textrm{Stein}}_F(\mu(t))
        =
        \int \int  \nabla\dFdmu(x) \ k(x,y)\ \nabla\dFdmu(y)
        \dd \mu(x)
        \dd \mu(y)
        .
        \label{eq:energy-dissipation-kernelized-Stein}
    \end{align}
    For the $\varphi$-divergence energy, \ie $F(\mu)=\mathrm{D}_\varphi(\mu |\pi)$, 
    \begin{multline}
        \calI^{\operatorname{Stein}}_\varphi(\mu|\pi)
        =
        \int\int
        \frac{\dd \mu}{\dd \pi }(x)
        \nabla \varphi^\prime\left(\frac{\dd \mu}{\dd \pi }(x)\right)
        k(x,y)
        \frac{\dd \mu}{\dd \pi }(y)
        \nabla \varphi^\prime\left(\frac{\dd \mu}{\dd \pi }(y)\right)
        \dd \pi (x)
        \dd \pi (y)
        .
    \end{multline}

Specifically, we find
\begin{enumerate}[itemsep=0pt, parsep=0pt, topsep=0pt, partopsep=0pt]
    \item the inclusive KL energy $F(\mu)=\mathrm{D}_{\varphi_0}(\pi|\mu)$,
    we obtain
     $\ksd^2(\pi|\mu )$, \ie squared reverse KSD
    \begin{align}
        \calI_0^{\mathrm{Stein}}(\mu|\pi)
        =
        -
        \int
        \int
        \nabla \log \frac{\dd \mu}{\dd \pi }(x)
        k(x,y)
        \nabla \log \frac{\dd \mu}{\dd \pi }(y)
        \dd \pi (x)
        \dd \pi (y)
        .
        \label{eq:reverse-KL-energy}
    \end{align}
    \item the squared Fisher-Rao energy $F(\mu)=\FR^2(\mu, \pi)$,
    \begin{align}
        \calI_{\frac12}^{\mathrm{Stein}}(\mu|\pi)
        =
        -
        4
        \int
        \int
        \nabla\sqrt{\frac{\dd \mu}{\dd \pi }(x)}
        k(x,y)
        \nabla\sqrt{\frac{\dd \mu}{\dd \pi }(y)}
        \dd \pi (x)
        \dd \pi (y)
        .
        \label{eq:squared-Fisher-Rao-energy}
    \end{align}
    \item the KL-entropy energy $F(\mu)=\mathrm{D}_{\mathrm{KL}}(\mu|\pi)$,
    we obtain $\ksd^2(\mu|\pi )$, a.k.a. the Stein-Fisher information
    \begin{align}
        \calI_1^{\mathrm{Stein}}(\mu|\pi)
        =
        -
        \int
        \int
        \nabla \log \frac{\dd \mu}{\dd \pi }(x)
        k(x,y)
        \nabla \log \frac{\dd \mu}{\dd \pi }(y)
        \dd \mu (x)
        \dd \mu (y)
        .
        \label{eq:KL-entropy-energy}
    \end{align}
    \item the $\chi^2$-divergence energy $F(\mu)=\mathrm{D}_{\chi^2}(\mu|\pi)$,
    \begin{align}
        \calI_2^{\mathrm{Stein}}(\mu|\pi)
        =
        -
        \int
        \int
        \nabla \frac{\dd \mu}{\dd \pi }(x)
        k(x,y)
        \nabla \frac{\dd \mu}{\dd \pi }(y)
        \dd \mu (x)
        \dd \mu (y)
        .
        \label{eq:chi-square-divergence-energy}
    \end{align}    
\end{enumerate}
\end{proposition}
Furthermore,
we observe a new connection between the MMD and the KSD
-- they are both generated by the inclusive KL ($0$-th order power entropy) dissipation in kernelized geometries.

\subsection*{Wasserstein-Fisher-Rao space and gradient flows}
\label{sec:wfr}
We now apply our framework to the
Wasserstein-Fisher-Rao
gradient flow,
also referred to as Hellinger-Kantorovich
by \citet{liero_optimal_2018} for a better accounting of the historical developments.
We first recall an elementary fact regarding duality of the inf-convolution of functionals.
\begin{lemma}
[Dissipation potential with inf-convolution]
\label{thm:dissipation-inf-conv}
Suppose the effective primal dissipation potential $\mathcal R$
is given by the inf-convolution of two dissipation potentials $\calR_1, \calR_2$, \ie
$\mathcal R(\mu, \cdot )  =\mathcal R_1(\mu,\cdot )  \square \mathcal R_2(\mu, \cdot ) $,
where $\square$ denotes inf-convolution.
The resulting gradient-flow equation
generated by the 
gradient system $(X, F, \calR)$
is given by
\begin{align}
    \label{eq:gfe-inf-conv}
    \dot \mu = \partial \mathcal R_1^*(\mu, -\mathrm{D}^X F)  + \partial \mathcal R_2^*(\mu, -\mathrm{D}^X F)
    ,
\end{align}
\end{lemma}
Note that 
    the two differentials $\mathrm{D}^X F$ in \eqref{eq:gfe-inf-conv}
    must be taken w.r.t. the same space $X$.

\begin{example}
[Hellinger-Kantorovich]

    In the setting of the Wasserstein-Fisher-Rao\\
    (Hellinger-Kantorovich) distance~\citep{liero_optimal_2018,chizat_unbalanced_2019},
    the gradient-flow equation~\eqref{eq:gfe-inf-conv}
    corresponds to
the reaction-diffusion PDE
\begin{align}
    \dot \mu = \mathcal R_{{W_2}}^*\left(\mu, - \dFdmu\right)  +    \partial \mathcal R_{D_\mathrm{Hellinger}}^*\left(\mu, - \dFdmu\right) 
    =  \DIV\left( \mu \nabla \dFdmu \right) -   \mu \dFdmu.
    \label{eq:react-diffus-pde}
\end{align}
\end{example}
\subsection*{Kernel approximation of the WFR gradient flow}
To produce a Stein-type geometry for the WFR distance,
we consider the primal and dual dissipation potentials as in Lemma~\ref{thm:dissipation-inf-conv},
where two dissipation potentials are obtained from the Stein and kernelized Fisher-Rao
\begin{align*}
\mathcal R_{\operatorname{K-WFR}}
=  \calR_{\mathrm{Stein}} \square \calR_{\FRk},
\quad
\calR^*_{\operatorname{K-WFR}}
    = \calR_{\mathrm{Stein}} + \calR_{\FRk}
    .
\end{align*}
Our starting point is therefore the kernelized reaction-diffusion equation
\begin{align}
    \dot \mu - \DIV\left(\mu\cdot\K_\mu \nabla \dFdmu\right)
    = - \mu\cdot\calS_\mu  \dFdmu
    .
    \label{eq:react-diffus-pde-kern}
\end{align}
where $\calS_\mu$ is the integral operator associated with the another kernel $s(\cdot, \cdot)$ that may be different from $k$.
We find
the following 
dynamic formulation of the \emph{kernelized Wasserstein-Fisher-Rao}
distance with the
kernerlized reaction-diffusion equation
\begin{multline}
        k\text{-}\WFR^2(\mu,\nu)
        =
        \min 
        \biggl\{\int_0^1
        \| \K_{u_t} \nabla \xi_t \|^2_{\rkhs}
        +
        \| \mathcal S_{u_t} \zeta_t\|^2_{\rkhs}
        \dd t
        \bigg|
        \\
        \dot {u_t} - \DIV\left({u_t}\cdot\K_{u_t} \nabla \xi_t\right)
        = - {u_t}\cdot \mathcal S_{u_t} \zeta_t
         ,
         u(0) =  \mu,
        u(1)=  \nu
        \biggr\}
        .
    \label{eq:KRR-react-diffus-kernelized}
\end{multline}

Going beyond kernelization,
we now construct the kernel-approximate WFR geometry
by considering both
inf-convolution and additive regularization
\begin{align}
\mathcal R_{\regWFR}(\rho, \cdot ) 
&:=
\mathcal R_{\operatorname{\Lstein}}(\rho, \cdot)
\square
\calR_{\LFRk}(\rho, \cdot)
,
\\
\mathcal R^*_{\regWFR}(\rho, \cdot ) 
&=
\mathcal R^*_{\operatorname{\Lstein}}(\rho, \cdot)
+
\calR_{\LFRk}^* (\rho, \cdot)
,
\label{eq:dissip-wfr-reg}
\end{align}
where $\mathcal R_{\operatorname{\Lstein}}(\rho, \cdot)$ is defined in \eqref{eq:stein-reg-dissipation-primal} and $\calR_{\LFRk}(\rho, \cdot)$ is defined in \eqref{eq:fr-kernel-reg-dual}.
We summarize the result below by performing the calculation using the inf-convolution rules
for both the kernelized version, i.e., Stein-type metric for the Wasserstein-Fisher-Rao case, and the
regularized kernel-Wasserstein-Fisher-Rao gradient flow.
\begin{corollary}
    [Kernel-approximate $\WFR$ gradient flow]
The generalized gradient system $(\Mplus,F, \WFR)$ generates the reaction-diffusion equation
where the velocity and growth field $v, r$ are given by the nonparametric regression solutions
\begin{align}
    \label{eq:KRR-react-diffus-krr}
    \begin{aligned}
        \dot {\mu_t} &= \DIV ({\mu_t}\cdot v_t) - {\mu_t}\cdot r_t , \\
        r_t &= \argmin_f \biggl\{ \left\|f -  \dFdmut\right\|^2_{L^2_{\mu_t}}  + \lambda \| f \|^2_{\rkhs} \biggr\},
        \
        v_t &= \argmin_f \biggl\{ \|f - \nabla \dFdmut\|^2_{L^2_{\mu_t}}  + \lambda \| f \|^2_{\rkhs} \biggr\}.
    \end{aligned}
\end{align}
\end{corollary}
\subsection{De-kernelized Wasserstein-Fisher-Rao and Kernel Sobolev-Fisher}
Similar to the De-Stein and MMD, we now consider the de-kernelized Wasserstein-Fisher-Rao (D-WFR) distance by de-kernelizing the reaction-diffusion equation~\eqref{eq:react-diffus-pde}
\begin{align}
    \dot \mu - \DIV\left(\mu\cdot\K_\mu^{-1} \nabla \dFdmu\right)
        = - \mu\cdot\K_\mu^{-1}  \dFdmu.
        \label{eq:KRR-react-diffus-DK}
\end{align}
We find the dynamic formulation and the regularized version
\begin{multline}
    \mathrm{D}\text{-}\WFR^2(\mu,\nu)
    =
    \inf 
    \biggl\{\int_0^1
    \|\nabla \xi_t \|^2_{\rkhs}
    +
    \| \zeta_t\|^2_{\rkhs}
    \dd t
    \bigg|
     \\
     \dot {u_t} - \DIV\left({u_t}\cdot\K_{u_t}^{-1} \nabla \xi_t\right)
    = - {u_t}\cdot \mathcal S_{u_t}^{-1} \zeta_t
     ,
     u(0) =  \mu,
    u(1)=  \nu
    \biggr\}
    ,
    \label{eq:D-WFR}
\end{multline}

The important insight here is:
\begin{proposition}
[De-kernelized Wasserstein-Fisher-Rao]
The De-kernelized Wasserstein-Fisher-Rao distance \eqref{eq:D-WFR} coincides with the inf-convolution of the MMD and the De-Stein distance, \ie
    $\operatorname{D-\WFR}^2( \mu,  \nu)
    :=
        \inf_{\pi\in \Mplus}
        \operatorname{De-Stein}^2( \mu,  \pi)
        \square
        \operatorname{MMD}^2( \pi, \nu)$.

    Furthermore, it admits the static dual formulation
    \begin{align}
        \operatorname{D-\WFR}^2( \mu,  \nu)
        =
        \sup_\zeta
        \left\{
            \int \zeta\dd (\mu- \nu)
            -\frac14
            \|\nabla \zeta\|^2_{\rkhs}
            -\frac14\|\zeta\|^2_{\rkhs}
            \right\}.
        \end{align}
    \label{thm:dk-wfr}
\end{proposition}
Compared with the MMD~\eqref{eq:dual-mmd-unconstrained}, the test function in the de-kernelized WFR is additionally regularized by the RKHS norm of the gradient, $\|\nabla \zeta\|^2_{\rkhs}$.

\subsection{Flattened Wasserstein-Fisher-Rao and kernel-Sobolev-Fisher}
Using the WFR-type inf-convolution
with the flattened Fisher-Rao (Section~\ref{sec:lin-fr}) and the flattened Wasserstein (Section~\ref{sec:lin-w}),
we obtain the following flattened Wasserstein-Fisher-Rao formulation
\begin{multline}
    \WFR_\omega^2(\mu,\nu)
    :=
    \min 
    \biggl\{\int_0^1
    \|\nabla \xi_t \|^2_{L^2_{\omega_1}}
    +
    \| \zeta_t\|^2_{L^2_{\omega2}}
    \dd t
    \bigg|
    \\
     \dot {u_t} - \DIV\left({\omega_1}\nabla \xi_t\right)
    = 
    -  {\omega_2}\zeta_t
     ,
     u(0) =  \mu,
    u(1)=  \nu\biggr\}
    ,
    \label{eq:bb-formula-ACH}
\end{multline}
The flattened geometry is the inf-convotion of the weighted $L^2_\omega$ and $H^{-1}_\omega$ norms, with the static formulation
\begin{align}
    \WFR_\omega^2(\mu,\nu)
    =
    \sup_\zeta
\left\{
\int \zeta\dd (\mu - \nu)
-\frac14
\|\nabla \zeta\|^2_{L^2_\omega}
-\frac14
\|\zeta\|^2_{L^2_\omega}
\right\}.
\end{align}
In the simplest case that $\omega_1,\omega_2$ are both the Lebesgue measure,
we recover the inf-convolution of the Allen-Cahn ($L^2$) and Cahn-Hilliard ($H^{-1}$).
Choosing $\omega_1=\omega_2=\nu$, this distance becomes the Sobolev-Fisher discrepancy~\citep{mrouehSobolevDescent2019}.
Furthermore, those authors proposed a regularized version with a heuristic RKHS norm regularization, of which we give the precise characterization.
\begin{example}
    [Regularized Kernel-Sobolev-Fisher discrepancy]
    The regularized kernel Sobolev-Fisher discrepancy proposed by~\citet{mroueh_unbalanced_2020}
    \begin{align}
        \mathrm{KSF}^2(\mu,\nu)
        =
        \sup_\zeta
\left\{
\int \zeta\dd (\mu - \nu)
-\frac14
\|\nabla \zeta\|^2_{L^2_\nu}
-\frac14
\|\zeta\|^2_{L^2_\nu}
-\frac{a}{2}\|\zeta\|^2_{\rkhs}
\right\},
\label{eq:heuristic-sobolev-static}
\end{align}
for $a>0$,
is equivalent to the following dynamic formulation
        \begin{align*}
            \min 
            \biggl\{\int_0^1
            \|\nabla \xi_t \|^2_{L^2_\nu}
            +
            \| \zeta_t\|^2_{L^2_\nu}
            +
            \frac1{2a}\|  \kappa_t \|^2_{\rkhs}
            \dd t
            \bigg|
             \dot {u_t} - \DIV\left(\nabla \xi_t\right)
            = 
            -  \zeta_t
            - \K^{-1} \kappa_t
             ,
             u(0) =  \mu,
            u(1)=  \nu
            \biggr\}
            .
        \end{align*}
    \end{example}
    It is
    the inf-convolution of three dissipaton geometries, the MMD, $H^{-1}_\nu$ norm, and $L^2_\nu$ norm.

\subsection*{Formal arguments for energy dissipation in kernel-approximate flows}
Related to the evolutionary $\Gamma$-convergence results in Section~\ref{sec:g-cvgs},
we now analyze the energy dissipatioin in the kernel-approximate gradientn flows.
First, we examine the kernel-regularized Fisher-Rao setting studied in Section~\ref{sec:kern-FR}.
For the RKHS approximation in this paper, we may use the standard approximation theory characterization, \eg \citep[Chapter~8]{cucker_learning_2007}, which was applied to the Stein setting by \citet{he2022regularized}. 
    One difference is that
    we do not make
    the assumption that $\dFdmu\in \operatorname{range}({\K_\mu}^s)$.
    In the Wasserstein or the Fisher-Rao setting, the regularity of $\dFdmu$ or $\nabla \dFdmu$ is determined by the gradient system and geodesic structure introduced in Section~\ref{sec:background}.
    It is not clear whether the range condition $\xi \in \operatorname{Range}(\K_\mu^{s})$ can be satisfied in any meaningful gradient systems.
    Furthermore, since $\K_\mu$ is a compact operator, hence the quantity $\|\K_\mu^{-\frac{s}{2}}\ \xi\|^2_{L^2_{\mu}}$ is unbounded.
Instead, we rely on regularization to avoid unbounded estimate.
We emphasize that the following arguments are merely formal and we provide
a rigorous justification in Section~\ref{sec:g-cvgs}.
\begin{proposition}
    [Energy dissipation of kernel-approximate Fisher-Rao flows]
    For $0\leq s \leq 1$, the energy dissipation satisfies
    \begin{align*}
        \frac{\dd}{\dd t} F(\mu(t)) 
        \leq
        \underbrace{-\left\|\dFdmut\right\|^2_{L^2_{\mu}}}_{\textrm{FR dissipation}}
        +
        \underbrace{{\lambda}^s
            \left\|
            \left(\K_\mu + \lambda \ID \right)^{-\frac{s}{2}}\dFdmut
            \right\|^2_{L^2_{\mu}}}_{{\textrm{approximation error}}}
            .
    \end{align*}
    \label{prop:local-Lojas-estimate-fine}
\end{proposition}
Note that the estimate $    \|
\left(\K_\mu + \lambda \ID \right)^{-\frac{s}{2}}\ \xi
\|^2_{L^2_{\mu}}$ is finite for any fixed $\lambda$.

The implication of the above estimate can be seen in the following scenario.
Suppose the function $F$ satisfies a Polyak-\Loj functional inequality in the pure Fisher-Rao geoemtry,
\ie
\begin{align}
    \left\|\dFdmu\right\|^2_{L^2_{\mu}}
    \geq 
    c\cdot
    \left(F(\mu) - \inf_\mu F(\mu)\right)
    .
    \label{eq:local-Lojas-FR}
\end{align}
Note that we have already justified the global \Loj inequality holds for the ${\phiP}$-divergence energy satisfying our threshold condition.
By the generalized Gr\"onwall's lemma,
the energy decay of the approximate Fisher-Rao gradient system satisfies the following
energy dissipation
estimate.
    \begin{multline}
        F(\mu(T)) - \inf_\mu F(\mu)
        \leq
        \underbrace{
         \rme^{-c\cdot T}
        \cdot
            \left(F(\mu(0)) - \inf_\mu F(\mu)\right)}_{\textrm{FR flow}}
        \\
        +
        \underbrace{\int_0^T
        \rme^{-c \cdot (T-t)}
        \cdot
        {\lambda}^s
            \|
            \left(\K_{\mu_t} + \lambda \ID \right)^{-\frac{s}{2}}\dFdmut
            \|^2_{L^2_{{\mu_t}}}
            \dd t}_{\textrm{approximation error}}
            ,
            \text{ for } 0\leq s \leq 1.
            \label{eq:local-Lojas-estimate-fine}
    \end{multline}
Those are merely formal arguments. They do not rigorously justify the asymptotic convergence behavior.
However,
based on our characterization of the entropy dissipation in the pure Fisher-Rao geometry,
we expect the approximation such as in Proposition~\ref{prop:local-Lojas-estimate-fine} to be close.
In Section~\ref{sec:g-cvgs}, we provide a rigorous justification, at the rigor level of applied analysis, of convergence as $\lambda\to 0$
for the kernel-approximate Fisher-Rao flows.

\subsection*{Other related works}
\label{sec:related-works}
There is a well-studied line of works using regularized energies of the Wasserstein gradient flow~\citep{carrillo2019blob} and Fisher-Rao gradient flow~\citep{lu2023birth}.
While closely related,
our work differs significantly in
1) the focus on modifying the dissipation geometry and the gradient structure instead of the energy -- our flows use the same energy as the target flow.
2) the focus on regression type approximation~\eqref{eq:gfe-krr-FRk} and \eqref{eq:KRR-stein} instead of relying on convolution and letting the kernel bandwidth go to zero.
Point 2) is important as it lets us use straightforward replacement with parametric models such as deep neural networks.
Furthermore, Section~\ref{sec:nonparametric-regression}
shows a direct connection between our regression formulation and the Rayleigh Principle for gradient flows.
\citet{daiProvableBayesianInference2016} proposed a particle mirror descent algorithm that can be thought as a practical numerical implementation of the Fisher-Rao gradient flow.
That paper also contains a kernelized particle approximation.
Related to kernel methods and Wasserstein flows,
In comparison, we have investigated the dynamics of generalized gradient flows (not necessarily Wasserstein) and their dissipation geometries, which should not be confused with the WGF with different energy objectives.
There also exist works such as
\citep{divol2022optimal} that characterizes the statistical error in the setting of approximating the optimal transport map using the static formulation when computing the Wasserstein distance.
\citet{marzouk_distribution_2023} also took a nonparametric regression perspective for ODE learning. Although their results do not concern gradient flows.
Additionally, there exists a large body of works that use gradient flows such as Langevin or birth-death dynamics to analyze neural network training, which we do not exhaust here.
There have also been a few recent works such as \citep{wangMeasureTransportKernel2024,mauraisSamplingUnitTime2024,nuskenSteinTransportBayesian2024}
concerning kernelization and gradient flows, whose perspectives complement ours.

\paragraph*{Acknowledgement}
This project was partially supported by
the Deutsche Forschungsgemeinschaft (DFG) through the Berlin Mathematics
Research Center MATH+ (EXC-2046/1, project ID: 390685689)
and
through the priority programme "Theoretical foundations of deep learning" (SPP 2298, project number: 543963649). 

\bibliography{ref}

\end{document}